\documentclass[preprint,12pt]{elsarticle}

\let\today\relax
\makeatletter
\def\ps@pprintTitle{%
    \let\@oddhead\@empty
    \let\@evenhead\@empty
    \def\@oddfoot{\footnotesize\itshape
         {Preprint (accepted by Neural Networks),  https://doi.org/10.1016/j.neunet.2025.107517} \hfill\today}%
    \let\@evenfoot\@oddfoot
    }
\makeatother

\usepackage{amssymb}
\usepackage{amsmath}
\usepackage{amsthm}
\usepackage{multirow}
\usepackage{slashbox}
\usepackage{hyperref}
\usepackage{url}
\usepackage{dsfont}
\usepackage{float}
\newtheorem{theorem}{Theorem}
\newtheorem{example}{Example}
\newtheorem{remark}{Remark}
\newtheorem{lemma}{Lemma}

\begin{document}

\begin{frontmatter}



\title{The informativeness of the gradient revisited}


\author{Rustem Takhanov} 

\affiliation{organization={Department of Mathematics, Nazarbayev University},
            addressline={53 Kabanbay Batyr Ave.}, 
            city={Astana},
            postcode={010000}, 
            country={Kazakhstan}}

\begin{abstract}
In the past decade gradient-based deep learning has revolutionized several applications. However, this rapid advancement has highlighted the need for a deeper theoretical understanding of its limitations. Research has shown that, in many practical learning tasks, the information contained in the gradient is so minimal that gradient-based methods require an exceedingly large number of iterations to achieve success. The informativeness of the gradient is typically measured by its variance with respect to the random selection of a target function from a hypothesis class. 

We use this framework and give a general bound on the variance in terms of a parameter related to the pairwise independence of the target function class and the collision entropy of the input distribution. Our bound scales as $ \tilde{\mathcal{O}}(\varepsilon+e^{-\frac{1}{2}\mathcal{E}_c}) $, where $ \tilde{\mathcal{O}} $ hides factors related to the regularity of the learning model and the loss function, $ \varepsilon $ measures the pairwise independence of the target function class and $\mathcal{E}_c$ is the collision entropy of the input distribution. 

To demonstrate the practical utility of our bound, we apply it to the class of Learning with Errors (LWE) mappings and high-frequency functions. In addition to the theoretical analysis, we present experiments to understand better the nature of recent deep learning-based attacks on LWE.
\end{abstract}



\begin{keyword}
  variance of the gradient \sep barren plateau \sep informativeness of the gradient \sep learning with errors



\end{keyword}

\end{frontmatter}



\section{Introduction}


Gradient-based learning has demonstrated remarkable success across a wide range of fields, including language modeling~\cite{10.1145/3641289}, protein structure prediction~\cite{abramson2024accurate}, reinforcement learning~\cite{wurman2022outracing}, and quantum chemistry~\cite{Keith2021}, among others. Recently, the application of gradient-based methods to the solution of combinatorial problems, and even cryptographic attacks, has gained significant attention~\cite{JMLR:v24:22-0567,NEURIPS2023_9bae70d3,pmlr-v202-geisler23a}.
Therefore, it is of interest to rigorously examine its limitations. Understanding the conditions under which gradient descent fails or performs suboptimally is crucial for improving its robustness across diverse tasks.
In our paper, we study the difficulties that gradient-based methods face when learning a target function drawn from a collection of functions that possess an almost pairwise independence property.

{\bf Gradient-based learning} refers to a broad category of learning algorithms that operate by minimizing an objective function through access to the approximate gradient at various points. This category includes widely used optimization techniques in deep learning, such as  Adam, AdaGrad, RMSProp, Nesterov Accelerated Gradient, SGD and others. 
A framework that encompasses these methods was proposed in~\cite{DBLP:journals/jmlr/Shamir18}, which offers a way to explain the phenomenon of {\bf low informativeness of the gradient}. 

We will briefly outline what is informativeness of the gradient in learning. 
Let us consider functions parameterized by some parameter $\alpha\in \mathcal{A}$ which we assume to be a random variable. The target function is supposed to be chosen from this set by sampling $\alpha$. Then, the variance of the gradient of supervised learning task's cost function at a fixed point, with respect to a random selection of $\alpha$, reflects the sensitivity of the gradient to the choice of the target function.
If the target function depends significantly on $\alpha$ but the gradient variance is negligible, then the outcome of the gradient-based optimization is likely independent of $\alpha$, making it unlikely that the method will effectively learn the target. If one observes such a situation, it is said that {\em the gradient is not informative}.

So far, the only way to verify the presence of this phenomenon is to establish an upper bound on the variance of the gradient in terms of parameters whose value scales can be controlled in advance. 
Generally, this type of upper bound includes two factors, say $A$ and $B$, where $B$ is negligibly small, and the $A$ term on the right-hand side quantifies a model's regularity (for example, of a function class used to fit the data, such as a neural network). The  $A$ term may restrict the framework's generality, as it overlooks the potential for non-regular models to still effectively learn the target function.
Upper bounds of this kind have been proven for classes of parity functions~\cite{DBLP:conf/icml/Shalev-ShwartzS17} and high-frequency functions~\cite{DBLP:journals/jmlr/Shamir18}. Similar results have emerged in quantum networks, where the term ``barren plateau phenomenon'' was introduced to describe this scenario~\cite{McClean2018,PhysRevLett.129.270501}. In our work, we establish a quite general inequality that bounds the gradient variance in terms of a measure of almost pairwise independence of the target function class and the regularity of a model.

{\bf Pairwise independence} is a characteristic of a collection of functions $\mathcal{H}$ between two probability spaces, $(\mathcal{X}, \mathcal{F}_{\mathcal{X}}, \mu_\mathcal{X})$ and $(\mathcal{Y}, \mathcal{F}_{\mathcal{Y}}, \mu_\mathcal{Y})$, and a distribution $\chi$ on $\mathcal{H}$. This property states that for any two distinct elements $x, y \in \mathcal{X}$, and a function $h \in \mathcal{H}$ sampled from the distribution $\chi$, the pair $[h(x), h(y)]$ is distributed according to $\mu_\mathcal{Y}\times \mu_\mathcal{Y}$. The concept was first introduced in cryptography~\cite{CARTER1979143} and has since been applied in areas such as message authentication~\cite{Daemen} and derandomization~\cite{TCS-010}. We us a relaxed form of this definition based on the Intergal Probability Metric between the distribution of the random variable $[h(x), h(y)]$ and the distribution $\mu_\mathcal{Y}\times \mu_\mathcal{Y}$. If the latter metric is $\mathcal{O}(\varepsilon)$ for some negligibly small $\varepsilon$, 
we refer to such function classes (equipped with a distribution $\chi$) as {\em almost pairwise independent} (see Section~\ref{nearly} for details).

In our analysis, we apply the gradient-based learning framework from~\cite{DBLP:journals/jmlr/Shamir18} to an almost pairwise independent target function class. The target function $ h \in \mathcal{H} $, sampled from $\chi$ and according to which the dataset $\{(x_i, h(x_i))\}_{i=1}^m$ was created, acts as the key parameter $\alpha$ mentioned earlier.
We establish an upper bound on the gradient variance (with respect to the randomness in choosing $ h $) of an objective function represented as $$\mathbb{E}_{X \sim \mu_{\mathcal{X}}}\Big[L\big(p(\mathbf{w}, X), h(X)\big)\Big],$$ where $ L $ is a loss function and $ p(\mathbf{w}, \cdot): \mathcal{X} \to \mathbb{R} $ is a set of functions used to fit data. Our {\em general variance bound} scales as $ \tilde{\mathcal{O}}(\varepsilon+e^{-\frac{1}{2}\mathcal{E}_c(\mu_\mathcal{X})}) $, where $ \tilde{\mathcal{O}} $ conceals factors related to the regularity of the set $\{p(\mathbf{w}, \cdot)\}$ and of the loss function $L$, $ \varepsilon $ measures the pairwise independence of $(\mathcal{H}, \chi)$, and $\mathcal{E}_c(\mu_\mathcal{X})$ is the collision entropy of the input distribution $\mu_{\mathcal{X}}$.

Note that our major upper bound, given in Section~\ref{main-upper-bounds} and proved in Appendix~\ref{aes-bound}, is {\em distribution-specific}. This is due to the fact that the measure of almost pairwise independence that we use, depends on distributions over inputs, outputs and the hypothesis set. Thus, using that bound a low informativeness of the gradient can be verified, but in principle, it does not rule out the possibility that the gradient {\em can become informative} if we change our distributions, especially, if we change a distribution over inputs. 

To illustrate practical implications of our general bound, we examine special cases of hypothesis sets. Notably, almost pairwise independent classes include Learning with Errors mappings and high-frequency functions as important examples. Let us briefly describe these latter function classes, along with the contexts in which they naturally appear and consequences of our bound for the learnability of these function classes.

\subsection{Examples of almost pairwise independent classes}
An instance of the {\bf Learning with Errors (LWE)} task is a pair $(A, \mathbf{b})$, where $A \in \mathbb{Z}_q^{m \times n}$ and $\mathbf{b} \in \mathbb{Z}_q^m$. In this scenario, $A$ is assumed to be uniformly generated from $\mathbb{Z}_q^{m \times n}$, and $\mathbf{b} = A\mathbf{s} + \mathbf{e}$, where $\mathbf{s}$ is a vector uniformly drawn from $\mathbb{Z}_q^n$ (called a secret), and $\mathbf{e} \in \mathbb{Z}_q^m$ is a vector with entries independently sampled from a fixed distribution $\kappa$ (typically, $\kappa$ is a discretized Gaussian distribution with a mean of zero, and $\mathbf{e}$ is referred to as the noise). The main objective in an LWE problem is to deduce the secret vector $\mathbf{s}$ from the pair $(A, \mathbf{b})$. If we express $A^\top = [\mathbf{a}_1, \cdots, \mathbf{a}_m]$ and $\mathbf{b} = [b_i]_{i=1}^m$, an LWE instance can be represented as the training set $\mathcal{T} = \{(\mathbf{a}_i, b_i)\}_{i=1}^m$, where $b_i = \langle \mathbf{s}, \mathbf{a}_i\rangle + e_i$, with $e_i \sim \kappa$. Consequently, we can formulate a slightly less strict version of the LWE problem, where the task is to learn the function $\mathbf{x} \mapsto \langle \mathbf{s}, \mathbf{x}\rangle$ over its entire domain $\mathbb{Z}_q^n$, given the set $\mathcal{T}$. 

The LWE task is a famous hard learning problem. There exists a standard reduction of the LWE task to the problem of finding the shortest vector in a given lattice (the exact-SVP problem). A well-known polynomial quantum reduction of approximate Shortest Vector Problem to Learning With Errors~\cite{10.1145/1060590.1060603}, along with a polynomial classical reduction~\cite{10.1145/2488608.2488680}, indicates that any polynomial-time algorithm for LWE could have significant implications. Currently, all algorithms for LWE that utilize a polynomial number of samples have an exponential asymptotic runtime~\cite{Guo2023}. Note that recently standardized protocols for post-quantum cryptography are based on the mathematical hardness of LWE~\cite{pqc}. While it is unlikely that a gradient-based method can solve LWE, it is practically important to estimate the maximum size of an LWE instance that could be potentially managed using such an approach. Our main bound is helpful in making such an estimate, so let us describe its consequences.

If all distributions (over inputs, outputs and the hypothesis set) are uniform, the set of functions $\mathbf{x} \mapsto \langle \mathbf{s}, \mathbf{x}\rangle$, parameterized by secret vectors, forms an almost pairwise independent class of functions. An application of our bound to this case yields that the variance of the gradient behaves approximately like $\tilde{\mathcal{O}}(q^{-\frac{n}{2}+1}) = \tilde{\mathcal{O}}(q e^{-\frac{n\log q}{2}})$. For the ``hardness'' parameter $n\log q$ from real applications, this value is extremely small, e.g. $10^{-800}$. This guarantees that the LWE task is resilient to straightforward attacks, based on the approximation of the LWE mapping by a neural network.

The situation changes for a non-uniform distribution over inputs. 
Specifically, the right hand side of our upper bound is heavily influenced by the collision entropy of the distribution from which inputs to an LWE mapping are drawn. E.g., if input vectors are sampled uniformly from  $({\mathbb{Z}}\cap [0,a))^n\setminus \{{\mathbf 0}\}$ where $a<q$, then our main bound implies that the variance of the gradient behaves approximately like $\tilde{\mathcal{O}}(q a^{-\frac{n}{2}})$. In other words, anyone who manages to compute input-output pairs where inputs are drawn from a distribution with low collision entropy may potentially solve the LWE task. We experimentally verify this conclusion for the case where the distribution on the hypothesis set (i.e. the distribution over secret vectors ${\mathbf s}$) is not uniform (in Section~\ref{experiments} we specifically analyze two cases that are popular in literature, namely sparse binary and sparse ternary secrets). 
This conclusion is in line with recent attempts to attack LWE by gradient-based learning methods with a preprocessing of the LWE task based on computing such ``low collision entropy'' pairs~\cite{DBLP:conf/nips/WengerCCL22,DBLP:journals/iacr/LiSWACL23,10.1145/3576915.3623076}. To our knowledge, our results provide the first theoretical analysis of these types of attacks on LWE.

{\bf High-frequency functions} are another class of functions that is a suitable testing ground for demonstrating consequences of our general bound. Given a 1-periodic function $h: {\mathbb R}\to {\mathbb R}$, this class consists of functions of the form $h({\mathbf w}^\top {\mathbf x})$ where ${\mathbf w}\in {\mathbb R}^n$ and $n$ is fixed. 
The informativeness of the gradient for this class of functions was studied in~\cite{DBLP:journals/jmlr/Shamir18}, where it was shown that the gradient's variance is small if (a) the objective is a mean squared error, (b) the distribution over the hypothesis set $\{h({\mathbf w}^\top {\mathbf x})| {\mathbf w}\in {\mathbb R}^n\}$ is induced by a uniform distribution over a sphere in ${\mathbb R}^n$ for the ``frequency'' vector ${\mathbf w}$, and (c) the distribution over inputs satisfies a certain Fourier transform-based condition.

We omit conditions (a) and (c) and, in place of the condition (b), study this class under a distribution over the hypothesis set induced by a central normal distribution in ${\mathbb R}^n$ with a covariance matrix $R \cdot I_n$. It turns out that this class of functions is almost pairwise independent, if the variance $R$ is large. Our main bound for the case $n\geq 2$ shows that the gradient's variance is limited by the product of three factors: (1) $\frac{1}{R^2}$, (2) a factor that measures the regularity of the model and loss function, (3) a certain integral depending only on the input distribution, $\mu_\mathcal{X}$. Thus, we observe a situation similar to the one in LWE: a class of target functions with a large norm of the ``frequency'' vector (controlled by the parameter $R$) is almost pairwise independent, but the gradient can be informative if the inputs are drawn from a suitable distribution. The case of $n=1$ leads to the same conclusion, though its analysis is trickier.

{\bf Organization.} In Section~\ref{nearly}, we provide a precise definition of the concept of an almost pairwise independent function class. Section~\ref{framework} outlines the general framework for gradient-based optimization introduced in~\cite{DBLP:journals/jmlr/Shamir18}. Our main upper bound on the gradient's variance is presented in Section~\ref{main-upper-bounds}. Section~\ref{LWE-case} is focused specifically on the Learning With Errors (LWE) case. In Section~\ref{experiments} we numerically estimate almost pairwise independence of the LWE hypothesis set under certain non-uniform distributions over secrets which play an important role in cryptographic applications and discuss the consequences of our bounds in the context of recent gradient-based attacks on LWE. In Section~\ref{waves} we discuss the implications of our main bound for the case of high-frequency functions. 

{\bf Notations.} Given a probability space $(\Omega, \mathcal{F}_{\Omega}, \nu)$, the notation $X \sim \nu$ means that the random variable $X$ is sampled from $\Omega$ according to the distribution $\nu$. For a measurable function $f: \Omega\to {\mathbb R}$, ${\mathbb E}_{X\sim \nu}[f(X)]$ denotes the intergal of $f$ w.r.t. $\nu$, i.e. $\int_{\Omega}f(x)d\nu(x)$. The space of measurable functions $f:\Omega\to {\mathbb R}$ such that $\int_{\Omega}|f|^pd\nu < \infty$ is denoted by $L_p(\Omega, \nu)$. Occasionally, when dealing with functions $f, g\in L_2(\Omega, \nu)$, $\langle f, g \rangle_\nu$ denotes the inner product $\int_{\Omega} f(x) g(x)d\nu(x)$. Accordingly, $\|f\|_{\nu} = \sqrt{\langle f, f \rangle_{\nu}}$. The completion of $(\Omega, \mathcal{F}_{\Omega}, \nu)$ is denoted by $(\overline{\Omega}, \overline{\mathcal{F}_{\Omega}}, \overline{\nu})$. We have $\Omega = \overline{\Omega}$, but when we use the symbol $\overline{\Omega}$ we mean that the set $\Omega$ is equipped with the completed $\sigma$-algebra and measure. For $(\Omega_i, \mathcal{F}_{i}, \nu_i)$, $i=1,2$, $(\Omega_1\times \Omega_2, \mathcal{F}_{1}\times \mathcal{F}_{2}, \nu_1\times \nu_2)$ denotes the product measure space. The extended real line $\{-\infty, \infty\}\cup {\mathbb R}$ is denoted by $\overline{{\mathbb R}} $.  For any finite set $S$, $|S|$ denotes its cardinality.  For a prime number $q$, $\mathbb{Z}_q$ denotes the quotient ring of 
$\mathbb{Z}/q\mathbb{Z}$. For vectors $\mathbf{x} = [x_i]_1^n$ and $\mathbf{y} = [y_i]_1^n \in \mathbb{Z}_q^n$, the inner product $\langle \mathbf{x}, \mathbf{y}\rangle$ is defined as $x_1 \cdot y_1 + \cdots + x_n \cdot y_n \in \mathbb{Z}_q$. 
For functions $f: U \to \mathbb{R}$ and $g: U \to \mathbb{R}_+$, we use the notation $f \lesssim g$ to indicate that a universal constant $\alpha \in \mathbb{R}_+$ exists such that $ |f(x)| \leq \alpha g(x) $ holds for every $x \in U$. The abbreviation a.e. (or, $\mu$-a.e.) means almost everywhere.

\section{Measuring almost pairwise independence by the Integral Probability Metric}\label{nearly}
Let $\mathcal{X}$ and $\mathcal{Y}$ be two sets, a set of inputs and outputs correspondingly. Let $(\mathcal{X}, \mathcal{F}_{\mathcal{X}}, \mu_{\mathcal{X}})$ and $(\mathcal{Y}, \mathcal{F}_{\mathcal{Y}}, \mu_{\mathcal{Y}})$ be probability spaces.  We are also given a hypothesis set $\mathcal{H}\subseteq \mathcal{Y}^{\mathcal{X}}$. 
We assume that there is unknown $h\in \mathcal{H}$ for which inputs $x_1, x_2, \cdots$ are sampled according to the probability measure $\mu_{\mathcal{X}}$, and the training data is $\{(x_1,h(x_1)), (x_2,h(x_2)), \cdots\}$. The role of the probability measure $\mu_{\mathcal{Y}}$ is different and this distribution does not affect the generation of the training set. 
 
Our theory is based on an analysis of the sensitivity of the gradient to a random choice of the hypothesis function $h\in \mathcal{H}$. Therefore, we additionally assume that $h$ is sampled from a distribution $\chi$, where $(\mathcal{H}, \mathcal{F}_{\mathcal{H}}, \chi)$ is a probability space and the mapping $(x,h)\to h(x)$ is a measurable function between the product measure space $\mathcal{X}\times \mathcal{H}$ and $\mathcal{Y}$.

The pair $(\mathcal{H}, \chi)$ is said to be pairwise independent (with the distribution of outputs $\mu_{\mathcal{Y}}$) if for any distinct $x, x' \in \mathcal{X}$ the random variable $[h(x), h(x')]$, where $h\sim \chi$, has the same distribution as the random variable $[Y,Y']$, where $Y,Y'\sim^{\rm iid}\mu_{\mathcal{Y}}$. If $\mathcal{X}$, $\mathcal{Y}$ are finite and $\mu_{\mathcal{X}}$, $\mu_{\mathcal{Y}}$, $\chi$ are uniform distributions, such families of mappings are commonly used in message authentication and universal hashing~\cite{CARTER1979143}. 
A well-known example of a pairwise independent family is the collection $\mathcal{H} = \{ h_{a,b} : \text{GF}(p^n) \to \text{GF}(p^n) \mid h_{a,b}(x) = ax + b, \; a, b \in \text{GF}(p^n) \}$, where $\text{GF}(p^n)$ represents the Galois field containing $ p^n $ elements. Further examples are provided in \cite{10.5555/1111205}.

Pairwise independence is ideally satisfied only for very special hypothesis sets and distributions. We measure almost pairwise independence using the Integral Probability Metric between distributions. 
Let $\mathcal{F} = (\mathcal{F}_1,\mathcal{F}_2)$, $\mathcal{F}_1\subseteq {\mathbb R}^{\mathcal{Y}}$ and $\mathcal{F}_2\subseteq {\mathbb R}^{\mathcal{Y}\times \mathcal{Y}}$ be a pair of functional spaces equipped with semi-norms $\|\cdot\|_{\mathcal{F}_1}$ and $\|\cdot\|_{\mathcal{F}_2}$ respectively, containing only measurable functions on $\mathcal{Y}$ and the product measure space $\mathcal{Y}\times \mathcal{Y}$. 
For a general non-finite case we measure the deviation of the pair $(\mathcal{H}, \chi)$ from pairwise independence with the function $\varepsilon_{\mathcal{F}}: \mathcal{X}\times \mathcal{X}\to [0,+\infty]$, where
\begin{equation}\label{tv2}
\begin{split}
&\varepsilon_\mathcal{F}(x, x') =
\sup_{f\in \mathcal{F}_2: \|f\|_{\mathcal{F}_2}\leq 1}
|{\mathbb E}_{h\sim \chi}[f(h(x), h(x'))]-
{\mathbb E}_{(Y, Y')\sim \mu^2_{\mathcal{Y}}}[f(Y,Y')]|,
\end{split}
\end{equation}
for $x \neq x'$, and $\varepsilon_\mathcal{F}(x, x) = \sup_{f\in \mathcal{F}_1: \|f\|_{\mathcal{F}_1}\leq 1}
|{\mathbb E}_{h\sim \chi}[f(h(x))]-{\mathbb E}_{Y\sim \mu_{\mathcal{Y}}}[f(Y)]|$. 
By construction, if $(\mathcal{H}, \chi)$ is pairwise independent, then $\varepsilon_\mathcal{F}(x, x') = 0$ for any $x,x'\in \mathcal{X}$. The latter Integral Probability Metric (ILP) is quite common in probabilistic applications~\cite{e999eb1a-76a6-3d4a-b9e6-48809754773a,10.1214/12-EJS722}. Specifically, we will use the following  functional spaces as $\mathcal{F}$.

\begin{example}
If we set $\mathcal{F} = \big(L_\infty (\mathcal{Y}, \mu_{\mathcal{Y}}), L_\infty (\mathcal{Y}\times \mathcal{Y}, \mu_{\mathcal{Y}}\times \mu_{\mathcal{Y}})\big)$, then
 $\varepsilon_\mathcal{F}(x, x')$ is twice the total variation distance between the random variables $[h(x), h(x')], h\sim \chi$ and $[Y,Y'], Y,Y'\sim^{\rm iid} \mu_{\mathcal{Y}}$. Analogously, $\varepsilon_\mathcal{F}(x, x)$ is twice the total variation distance between the random variables $h(x), h\sim \chi$ and $Y\sim \mu_{\mathcal{Y}}$.
\end{example}

\begin{example}\label{Pearson}
If $|\mathcal{Y}|<\infty$ and $p_y = \mu_{\mathcal{Y}}(\{y\}), y\in \mathcal{Y}$, one can set $\mathcal{F} = \big(L_2 (\mathcal{Y}, \mu_{\mathcal{Y}}), L_2 (\mathcal{Y}\times \mathcal{Y}, \mu_{\mathcal{Y}}\times \mu_{\mathcal{Y}})\big)$. Then
 $\varepsilon_\mathcal{F}(x, x')$ is the so-called  Pearson $\chi^2$ divergence that is equal to
\begin{equation*}
\begin{split}
&\varepsilon_\mathcal{F}(x, x') = \sqrt{\sum_{y,y'\in \mathcal{Y}}\frac{({\mathbb P}_{h\sim \chi}[h(x)=y, h(x')=y']-p_yp_{y'})^2}{p_yp_{y'}}},
\end{split}
\end{equation*}
for $x\ne x'$, and 
\begin{equation*}
\begin{split}
\varepsilon_\mathcal{F}(x, x) = 
\sqrt{\sum_{y\in \mathcal{Y}}\frac{({\mathbb P}_{h\sim \chi}[h(x)=y]-p_y)^2}{p_y}}.
\end{split}
\end{equation*}
For details on the Pearson $\chi^2$ divergence see Remark 1 in~\cite{10.5555/3294996.3295012} and Table 1 in~\cite{mroueh2018sobolev}.
\end{example}

\section{Gradient-based optimization}\label{framework}
Recall that $h$ is sampled from $\chi$. 
To approximate the mapping $x \to h(x)$, we use a specified collection of functions from $\mathcal{X}$ to $\mathbb{R}$, parameterized by a weight vector $\mathbf{w} \in O \subseteq \mathbb{R}^{N_{\rm par}}$. This family is represented as $\{ p(\mathbf{w}, \cdot) : \mathcal{X} \to \mathbb{R} \mid \mathbf{w} \in O \}$.
Let $O\subseteq {\mathbb R}^{N_{\rm par}}$ be an open set (equipped with Lebesgue measure) and $p: O\times \mathcal{X}\to {\mathbb R}$ be a measurable function such that $\nabla_{\mathbf w}p({\mathbf w}, x)$ exists a.e. over $({\mathbf w}, x)\in O\times \mathcal{X}$ and is bounded.

Given a fixed loss function $L: {\mathbb R}\times \mathcal{Y}\to {\mathbb R}$, our goal is to minimize the following objective with respect to ${\mathbf w}$:
\begin{equation}\label{goal}
C_h({\mathbf w}) = \mathbb{E}_{X \sim \mu_{\mathcal{X}}}\Big[L\big(p(\mathbf{w}, X), h(X)\big)\Big].
\end{equation}
Our assumptions about $L$ are the most general.
Let $L: {\mathbb R}\times \mathcal{Y}\to {\mathbb R}$ be a measurable function such that: (a) $\frac{\partial L (p({\mathbf w}, x),y)}{\partial p}$ exists a.e. over $({\mathbf w}, x,y)$ and is bounded; (b)  $\frac{\partial L (p({\mathbf w}, x),h(x))}{\partial p}$ exists a.e. over $({\mathbf w}, x,h)$; (c) $L(p({\mathbf w},x), h(x))$ is bounded on $O\times \mathcal{X}\times \mathcal{H}$.
The latter assumptions are necessary for the cost function $C_h({\mathbf w})$ and the gradient of $C_h({\mathbf w})$ to be properly defined, making them essential but non-restrictive. Note that the gradient is given by $\nabla_{{\mathbf w}}C_h({\mathbf w}) = {\mathbb E}_{X\sim \mu_\mathcal{X}}\big[\frac{\partial L(p({\mathbf w}, X), h(X))}{\partial p}\nabla_{{\mathbf w}}p({\mathbf w}, X)\big]$ (this is proven thoroughly in appendix). Thus, the magnitude of the gradient vector is influenced by the size of the vector $\nabla_{{\mathbf w}}p({\mathbf w}, X)$ and the scalar $\frac{\partial L}{\partial p}(p({\mathbf w},  X), h(X))$.

We assume that the optimization problem~\eqref{goal} is solved using a gradient-based method. By this, we mean any method that iteratively generates points ${\mathbf w}_1, {\mathbf w}_2, \cdots \in O$ according to the rule ${\mathbf w}_{t+1}=f_t({\mathbf g}_1, \cdots, {\mathbf g}_t)$, where ${\mathbf g}_t$ is an approximation of the gradient $\nabla_{{\mathbf w}}C_h({\mathbf w}_t)$ and $\{f_t: {\mathbb R}^{tN_{\rm par}}\to {\mathbb R}^{N_{\rm par}}\}_{t=1}^\infty$ are some deterministic or randomized functions. The gradient approximation ${\mathbf g}_t$ is provided by an oracle ${\mathfrak O}$. We assume that ${\mathbf g}_t$ is stochastic and $\|{\mathbf g}_t-\nabla_{{\mathbf w}}C_h({\mathbf w}_t)\|< \delta$, with no guarantees beyond this accuracy level $\delta$. For this reason, the oracle is referred to as the {\em $\delta$-accurate gradient oracle}, making $\delta>0$ a crucial parameter in this framework. 
In practice, ${\mathbf g}_t$ is derived from the dataset, which consists of a set of random pairs $\{(x_i, h(x_i))\}$ that we have access to. Examples of gradient-based methods include Stochastic Gradient Descent (SGD), RMSProp, Nesterov Momentum, Adam, and other methods.

The gradient-based algorithm framework introduced in~\cite{DBLP:journals/jmlr/Shamir18} provided an approach to examining the limitations of such methods when learning functions of the form ${\mathbf x} \to \psi(\langle \mathbf{w}, {\mathbf x} \rangle)$ on ${\mathbb R}^n$, where $\psi$ is a 1-periodic function. Our goal is to adapt this framework to study the function $x \to h(x)$, defined over the domain $\mathcal{X}$. Consequently, we present an adapted version of Theorem 4 from~\cite{DBLP:journals/jmlr/Shamir18} as follows.

\begin{theorem}[\cite{DBLP:journals/jmlr/Shamir18}]\label{shamir}  
Suppose that $\delta > 0$ is chosen such that ${\rm Var}_{h \sim \chi}[\nabla C_h(\mathbf{w})] \leq \delta^3$ for any $\mathbf{w} \in O$. Then, a $\delta$-accurate gradient oracle can be defined, guaranteeing that for any algorithm of the specified type and any probability $p \in (0, 1)$, the algorithm's output will, with probability at least $1 - p$ over the choice of hypothesis $h$, become independent of $h$ after at most $\frac{p}{\delta}$ iterations.
\end{theorem}

In a gradient-based learning process, the output should not be independent of the hypothesis $h$. Therefore, if the accuracy parameter $\delta$ exceeds ${\rm Var}_{h \sim \chi}[\nabla C_h(\mathbf{w})]^{\frac{1}{3}}$, the number of iterations required to succeed in the task would need to be approximately $\mathcal{O}({\rm Var}_{h \sim \chi}[\nabla C_h(\mathbf{w})]^{-\frac{1}{3}})$. However, if ${\rm Var}_{h \sim \chi}[\nabla C_h(\mathbf{w})]$ is extremely small (for example, on the order of $10^{-60}$, as we will show), the algorithm will be fundamentally unable to succeed.

In the following section, we will present our upper bound for ${\rm Var}_{h \sim \chi}[\nabla C_h(\mathbf{w})]$.

\section{Upper bound on the variance}\label{main-upper-bounds}
Given the loss function $L$ and the family $\{p({\mathbf w}, \cdot)\}$, we denote
\begin{equation*}
\begin{split}
r_{\mathbf w}(x,y) = \frac{\partial L(p({\mathbf w},x), y)}{\partial p}-{\mathbb E}_{Y\sim \mu_\mathcal{Y}}\big[\frac{\partial L(p({\mathbf w},x), Y)}{\partial p}\big].
\end{split}
\end{equation*}
Also, the following value captures the typical deviation of $\frac{\partial L(p({\mathbf w}, x), Y)}{\partial p}$ from its expected value:
\begin{equation*}
\begin{split}
D_x = {\rm Var}_{Y\sim \mu_\mathcal{Y}}[\frac{\partial L(p({\mathbf w},x), Y)}{\partial p}].
\end{split}
\end{equation*}

The next theorem offers a general upper bound on the variance of the cost function gradient for a randomly chosen hypothesis $h$.

\begin{theorem}\label{aes-bound}  
We have
\begin{equation}\label{aes-uncertainty}
\begin{split}
&{\rm Var}_{h\sim\chi}\big[\partial_{w_i} {\mathbb E}_{X\sim \mu_{\mathcal{X}}} [L(p({\mathbf w},X), h(X))]\big]\leq \\ 
&\|\frac{\partial p({\mathbf w},\cdot)}{\partial w_i}\|^2_{\mu_{\mathcal{X}}}   \big(\sqrt{{\mathbb E}_{(X,Y)\sim \mu_{\mathcal{X}}^2}\big[ \varepsilon_\mathcal{F}(X,Y)^2\|\phi_{X,Y}\|_{\mathcal{F}_2}^2\big]} +\sqrt{\gamma} \big),
\end{split}
\end{equation}
where
\begin{equation*}
\begin{split}
\phi_{x,x'}(y,y')=r_{\mathbf w}(x,y) r_{\mathbf w}(x',y'),
\phi_{x}(y)=r_{\mathbf w}(x,y)^2 
\end{split}
\end{equation*}
and
\begin{equation*}
\begin{split}
&\gamma = {\mathbb P}_{(X,Y)\sim \mu_{\mathcal{X}}^2}[X=Y]\cdot {\mathbb E}_{(X,Y)\sim \mu_{\mathcal{X}}^2}[(\|\phi_{X}\|_{\mathcal{F}_1}\varepsilon_\mathcal{F}(X,X)+D_X)^2\mid X=Y].
\end{split}
\end{equation*}
\end{theorem}

\begin{remark} Let us consider a case where $\mathcal{F}_i = L_\infty(\mathcal{Y}^i,\mu_\mathcal{Y}^i)$, $i=1,2$ and the loss function $L$ satisfies $|L(p, y)-L(p', y)|\leq c|p-p'|$, i.e. $L$ is Lipschitz w.r.t. the first variable. Then we have $r_{\mathbf w}(x,y)\lesssim 1$, $D_x\lesssim 1$, $\|\phi_x\|_{\mathcal{F}_1}\lesssim 1$, $\|\phi_{x,x'}\|_{\mathcal{F}_2}\lesssim 1$ and $\varepsilon_\mathcal{F}(X,X)\lesssim 1$ (recall that $\varepsilon_\mathcal{F}(X,X)$ is twice the total variation distance). Theorem~\ref{aes-bound} directly implies 
\begin{equation}\label{RHS}
\begin{split}
&{\rm Var}_{h\sim \chi}\big[\partial_{w_i} C_h({\mathbf w})\big]\lesssim  \\
&{\mathbb E}_{X\sim \mu_\mathcal{X}}[(\frac{\partial p({\mathbf w},X)}{\partial w_i})^2] \big({\mathbb E}_{(X,Y)\sim \mu_\mathcal{X}^2}[ \varepsilon_\mathcal{F}(X,Y)^2]^{\frac{1}{2}}\hspace{-1pt}+\hspace{-1pt} {\mathbb P}_{(X,Y)\sim \mu_\mathcal{X}^2}[X\hspace{-1pt}=\hspace{-1pt}Y]^{\frac{1}{2}}\big).
\end{split}
\end{equation}
Recall that the collision entropy of a distribution $\mu_\mathcal{X}$, denoted by $\mathcal{E}_c(\mu_\mathcal{X})$, is the expression $$-\log({\mathbb P}_{(X,Y)\sim \mu_\mathcal{X}^2}[X=Y]).$$ 
Thus, ${\mathbb P}_{(X,Y)\sim \mu_\mathcal{X}^2}[X=Y] = e^{-\mathcal{E}_c(\mu_\mathcal{X})}$.
Then, the RHS of the bound~\eqref{aes-uncertainty} is proportional to 
\begin{equation*}
\begin{split}
&{\mathbb E}_{X\sim \mu_\mathcal{X}}[(\frac{\partial p({\mathbf w},X)}{\partial w_i})^2] \big({\mathbb E}_{(X,Y)\sim \mu_\mathcal{X}^2}[ \varepsilon_\mathcal{F}(X,Y)^2]^{\frac{1}{2}}+e^{-\frac{1}{2}\mathcal{E}_c(\mu_\mathcal{X})}\big).
\end{split}
\end{equation*}
The inequality demonstrates that three key factors affect the informativeness of the gradient. Specifically, the variance of the gradient w.r.t. a random choice of $h\sim \chi$ decreases as:
\begin{itemize}
\item The more pairwise independent $(\mathcal{H}, \chi)$ is w.r.t. to $\mu_{\mathcal{Y}}$ (expressed by the term ${\mathbb E}_{(X,Y)\sim \mu_\mathcal{X}^2}[ \varepsilon_\mathcal{F}(X,Y)^2]^{\frac{1}{2}}$).
\item The greater the collision entropy of $\mu_\mathcal{X}$ (captured by the term $e^{-\frac{1}{2}\mathcal{E}_c(\mu_\mathcal{X})}$).
\item The more regular the neural network's parameterization (represented by the term ${\mathbb E}_{X\sim \mu_\mathcal{X}}[(\frac{\partial p({\mathbf w},X)}{\partial w_i})^2] $).
\end{itemize}
If the collision entropy of $\mu_\mathcal{X}$ is large, one can neglect the second term. On the contrary, if $\mathcal{E}_c(\mu_\mathcal{X})$ is moderate, then the RHS of the inequality~\eqref{aes-uncertainty} may be large enough, which admits the possibility of a successful gradient-based training (though it does not guarantee it). In the next section, in the analysis of  LWE hypothesis sets we demonstrate that, indeed, the collision entropy of $\mu_\mathcal{X}$ strongly influences the informativeness of the gradient.
\end{remark}

\section{Application of Theorem~\ref{aes-bound} to attacks on LWE}\label{LWE-case}
To illustrate the relevance of Theorem~\ref{aes-bound}, we will now explore specific examples of hypothesis sets and distributions $\mu_\mathcal{X}$, $\mu_\mathcal{Y}$ and $\chi$. We will examine the LWE (Learning with Errors) hypothesis set which is defined by
\begin{equation}\label{LWE-H}
\begin{split}
&\mathcal{H} = \{h_{\mathbf{k}}: {\mathbb{Z}}_q^n\setminus \{{\mathbf 0}\} \to {\mathbb{Z}}_q \mid  h_{\mathbf{k}}({\mathbf{x}}) = \langle {\mathbf{k}}, {\mathbf{x}} \rangle, {\mathbf{k}} \in {\mathbb{Z}}_q^n \},
\end{split}
\end{equation}
where $ \langle {\mathbf{k}}, {\mathbf{x}} \rangle = k_1x_1 + \cdots + k_nx_n \mod q $ and $ q \geq 2 $ is prime. Note that we excluded ${\mathbf 0}$ from the domain $\mathcal{X}$, due to the fact that any hypothesis sends it to zero. The result presented here follows directly as a corollary of Theorem~\ref{aes-bound}.

\begin{theorem}\label{LWE-weak}
Let $ \mathcal{X} = {\mathbb{Z}}_q^n \setminus \{{\mathbf 0}\}$, $ \mathcal{Y} = {\mathbb{Z}}_q $ and $\mu_\mathcal{X}$ be the discrete uniform distribution on $({\mathbb{Z}}\cap [0,a))^n\setminus \{{\mathbf 0}\}$ where $a\in \{1,2,\cdots,q\}$, $\mu_\mathcal{Y}$ be the discrete uniform distribution on $\mathcal{Y}$. For the hypothesis set~\eqref{LWE-H} and the discrete uniform distribution $\chi$ on $\mathcal{H}$, we have
\begin{equation}\label{uncertainty-one}
\begin{split}
&{\rm Var}_{h \sim \chi} \left[ \partial_{w_i} {\mathbb{E}}_{X \sim \mu_\mathcal{X}} L(p({\mathbf{w}}, X), h(X)) \right] \leq \\
&{\mathbb{E}}_{X \sim \mu_\mathcal{X}} \left[(\frac{\partial p({\mathbf{w}}, X)}{\partial w_i})^2 \right]  {\mathbb{E}}_{X \sim \mu_\mathcal{X}} \left[ D_X^2 \right]^{\frac{1}{2}}   \left( \sqrt{(q+1)(q-2)}+1\right)  (a^{n}-1)^{-\frac{1}{2}}.
\end{split}
\end{equation}
\end{theorem}

\begin{remark}
First, let's examine the situation when $a = q$. In cryptographic schemes based on the LWE hypothesis set, typical values are $ \log_2 q \approx 10 $ and $ n \approx 544 $~\cite{10.1145/2976749.2978425}. As a result, the factor $(q^n-1)^{-\frac{1}{2}} \sim 10^{-800} $ on the RHS of~\eqref{uncertainty-one} is negligible. This suggests that for the gradient to be informative, one must carefully choose a neural network architecture and a loss function so that either ${\mathbb{E}}_{X \sim \mu_\mathcal{X}} \left[(\partial_{w_i} p({\mathbf{w}}, X))^2 \right]$ or ${\mathbb{E}}_{X \sim \mu_\mathcal{X}} \left[ D_X^2 \right]$ is sufficiently large. However, this makes the learning process impractical.
\end{remark}

\begin{remark} Now let $a<q$. In cryptographic applications, it is standard to evaluate the hardness of attacking a specific cryptographic primitive based on the number of bits (or nats) required to encode the secret key. Notably, the right-hand side of bound \eqref{uncertainty-one} behaves like $\tilde{\mathcal{O}}((a^{n}-1)^{-\frac{1}{2}})$, meaning that the "hardness of  learning parameter" is given by $\log(a^{n}-1) \approx n\log a$, rather than by $\log |\mathcal{H}| = n\log q$. Notebly, $n\log q > n\log a$. This suggests that the non-uniformity of the input distribution can significantly improve the learnability of the LWE hypothesis set. 
\end{remark}

\begin{remark}
Suppose now that we have access to (noisy) ``input-output'' pairs of an LWE mapping, where the input ${\mathbf x}$ is uniformly distributed over ${\mathbb Z}_q^n\setminus \{{\mathbf 0}\}$. A straightforward approach to form a training set is to repeat the following line: to sample a pair $({\mathbf x} = [x_i]_{i=1}^n, y)$ and put it into the training set if $0\leq x_i\leq a-1$ for any $i=1,\cdots, n$. 
Even if access to pairs is relatively ``cheap'', the latter approach fails due to the fact that adding an element into the training set requires generating $\mathcal{O}\big((\frac{q}{a})^n\big)$ pairs, which is prohibitively expensive. 

Therefore, obtaining pairs with a low variability of input's components requires a non-trivial pre-processing of initial pairs. One such approach was utilized in a recent attack on LWE~\cite{10.1145/3576915.3623076}.
In the following section, we experimentally study hypothesis sets with binary and ternary secrets, that is the case of LWE that was successfully attacked in the cited paper. 
\end{remark}

\section{Experiments with binary and ternary secrets}\label{experiments}
Let us now consider the case of sparse binary secrets, i.e. the case when the hypothesis set equals
\begin{equation}\label{LWE-B}
\begin{split}
&\mathcal{H}_{l}^{\rm b} = \{h_{\mathbf{k}}: {\mathbb{Z}}_q^n \to {\mathbb{Z}}_q \mid {\mathbf{k}} \in {\mathbb{Z}}_2^n, \|\mathbf{k}\|_{\rm H}\leq l,  h_{\mathbf{k}}({\mathbf{x}}) = \langle {\mathbf{k}}, {\mathbf{x}} \rangle \},
\end{split}
\end{equation}
where $\|\mathbf{k}\|_{\rm H}$ denotes the number of non-zero entries in $\mathbf{k}$ and $l\in {\mathbb N}$. The parameter $l$ is called the height. Analogously, for the ternary secrets case we define
\begin{equation}\label{LWE-T}
\begin{split}
&\mathcal{H}_{l}^{\rm t} = \{h_{\mathbf{k}}: {\mathbb{Z}}_q^n \to {\mathbb{Z}}_q \mid {\mathbf{k}} \in \{-1,0,1\}^n, \|\mathbf{k}\|_{\rm H}\leq l, h_{\mathbf{k}}({\mathbf{x}}) = \langle {\mathbf{k}}, {\mathbf{x}} \rangle \}.
\end{split}
\end{equation}
The latter set plays an important role in homomorphic encryption standards~\cite{albrecht2021homomorphic}.
By $\chi_{l}^{\rm b}$ (or, $\chi_{l}^{\rm t}$) we denote the discrete uniform distribution over $\mathcal{H}_{l}^{\rm b}$ ($\mathcal{H}_{l}^{\rm t}$). As in Theorem~\ref{LWE-weak}, we assume that $ \mathcal{X} = {\mathbb{Z}}_q^n \setminus \{{\mathbf 0}\}$, $ \mathcal{Y} = {\mathbb{Z}}_q $ and $\mu_\mathcal{X}$ is the discrete uniform distribution on $({\mathbb{Z}}\cap [0,a))^n\setminus \{{\mathbf 0}\}$ where $a\in \{1,2,\cdots,q\}$, $\mu_\mathcal{Y}$ is the discrete uniform distribution on $\mathcal{Y}$. 
Let us set $\mathcal{F} = (L_\infty(\mathcal{Y}, \mu_\mathcal{Y}), L_\infty(\mathcal{Y}^2, \mu^2_\mathcal{Y}))$ and study $\varepsilon_{\mathcal{F}}(\mathbf{x},\mathbf{x}')$ for both hypothesis sets~\eqref{LWE-B} and~\eqref{LWE-T} and  the corresponding uniform distributions $\chi_{l}^{\rm b}$ and $\chi_{l}^{\rm t}$. 

\begin{table}
\begin{minipage}[t]{.98\textwidth}
\centering
\scriptsize
\begin{tabular}{|c |c | c | c | c | } 
\hline
\multirow{3}{*}{$q$} & \multicolumn{2}{c|}{$R^2$} & \multicolumn{2}{c|}{$R^2$} \\ [0.4ex] 
\cline{2-5}
 & \multicolumn{2}{c|}{binary secrets} & \multicolumn{2}{c|}{ternary secrets} \\ [0.4ex] 
\cline{2-5}
 & without $\log a$ & with $\log a$  & without $\log a$ & with $\log a$ \\ [0.4ex] 
\hline
3 & 0.824 & {\bf 0.961} & 0.829 & {\bf 0.955} \\
\hline
5 & 0.737 & {\bf 0.853} & 0.807 & {\bf 0.909} \\
\hline
7 & 0.780 & {\bf 0.862} & 0.811 & {\bf 0.860} \\
\hline
\end{tabular}
\caption{The coefficient of determination of $-\log(\varepsilon)$ as a linear combination of $\log (|\mathcal{H}|)$ and $\log a$.}\label{R-squared}
\end{minipage}\\
\begin{minipage}[t]{.98\textwidth}
\centering
\scriptsize
\begin{tabular}{|l|llllllr|}
\hline
\backslashbox{$l/n$}{$a/q$} & 0.35 & 0.4 & 0.45 & 0.5 & 0.55 & 0.6 & 0.65 \\
\hline
0.16 & 1.0 & 1.0 & 1.0 & 1.0 & 1.0 & 1.0 & 0.88 \\
0.18 & 1.0 & 1.0 & 1.0 & 1.0 & 0.82 & 0.86 & 0.84 \\
0.20 & 1.0 & 1.0 & 1.0 & 1.0 & 1.0 & 0.82 & 0.82 \\
0.22 & 0.98 & 1.0 & 1.0 & 0.98 & 0.80 & 0.78 & 0.86 \\
0.24 & 1.0 & 1.0 & 1.0 & 0.98 & 0.78 & 0.78 & 0.80 \\
0.26 & 1.0 & 1.0 & 0.88 & 0.92 & 0.76 & 0.76 & 0.76 \\
0.28 & 0.98 & 1.0 & 0.80 & 0.74 & 0.74 & 0.76 & 0.74 \\
0.30 & 0.98 & 1.0 & 0.93 & 0.76 & 0.72 & 0.74 & 0.74 \\
\hline
\end{tabular}
\caption{From~\cite{DBLP:conf/nips/WengerCCL22}. The table shows the proportion of the secret recovered by SALSA for $n = 50$, given varying values of $\frac{l}{n}$ and $\frac{a}{q}$.}\label{table-from-salsa}    
\end{minipage}
\end{table}

Unfortunately, it is non-trivial to give an estimate that is analogous to~\eqref{uncertainty-one} for the latter two cases. Instead, we computed exactly the expression 
$$\varepsilon = {\mathbb E}_{(X,Y)\sim \mu_\mathcal{X}^2}[ \varepsilon_\mathcal{F}(X,Y)^2]^{\frac{1}{2}}$$
 for different values of parameters $n,l,a$ and prime numbers $q=3,5,7$ using OR-Tools, Google's open source software for combinatorial optimization~\cite{ortools}. Our computations show that the second term ${\mathbb P}_{(X,Y)\sim \mu_\mathcal{X}^2}[X=Y]^{\frac{1}{2}}$ in the RHS of the inequality~\eqref{RHS} can be neglected, in comparison with the first term.
Our goal was to verify that the logarithm of the RHS of the inequality~\eqref{RHS}  is mainly influenced by the logarithm of the cardinality of the hypothesis set (i.e. the number of nats needed to encode the secret) and the collision entropy of $\mu_{\mathcal{X}}$ (i.e. $\log(a^n-1)\approx n\log a$). Our experiments showed that $\log a$ is a better predictor of $-\log(\varepsilon)$ than $n\log a$. That is why we report results for the following two linear regression hypotheses:
\begin{equation*}
\begin{split}
-\log(\varepsilon) = \beta_0+\beta_1 \log(|\mathcal{H}|) +\theta,
\end{split}
\end{equation*}
i.e. without $\log a$, and
\begin{equation*}
\begin{split}
-\log(\varepsilon) = \beta_0+\beta_1 \log(|\mathcal{H}|)+\beta_2 \log a +\theta,\\
\end{split}
\end{equation*}
i.e. with $\log a$, where $\beta_0, \beta_1,\beta_2 $ are weights of features and $\theta$ is noise. The coefficient of determination, $R^2$, for both cases (and both binary and ternary secrets) are given in Table~\ref{R-squared}. As can be seen from the table, the term $\log a$ always substantially improves the $R^2$ coefficient, and together they predict $-\log(\varepsilon) $ with a pretty high accuracy. Scatter plots and linear regression lines for the pair $(-\log (\varepsilon), \log(|\mathcal{H}|))$, drawn for different values of $a=2,\cdots,q$ on Figure~\ref{lrplots}, also support this claim. Surprisingly, when we apply a linear regression model with all accessible predictors, i.e. when $-\log (\varepsilon)$ is treated as a linear combination of $\log(|\mathcal{H}|)$, $\log a$, $\log n$, $\log l$, only $\log(|\mathcal{H}|)$ and $\log a$ are recognized by the $t$-test as valuable features.

This indicates that when one has access to LWE samples with the distribution of inputs supported in $({\mathbb{Z}}\cap [0,a))^n\setminus \{{\mathbf 0}\}$, decreasing the value of $a$ can potentially improve the informativeness of the gradient. This finding is in line with the experimental results of studies on the learnability of the LWE mapping reported in~\cite{DBLP:conf/nips/WengerCCL22}. For completeness, we borrow Table~\ref{table-from-salsa} from the cited paper. 
The table illustrates the enhanced learnability achieved by sampling input vectors for the LWE mapping from $({\mathbb{Z}} \cap [0, a))^n \setminus {\mathbf{0}}$. When $n = 50$ and $\frac{a}{q} = 0.4$, the full secret is successfully recovered. In contrast, for $n = 50$ and $\frac{a}{q} = 0.65$, around 80\% of the secret can be recovered.

This remarkable phenomenon, i.e. the positive effect of the decrease in $a$ on the informativeness of the gradient, was recently leveraged in a promising attack on LWE~\cite{10.1145/3576915.3623076} with parameter sizes approaching NIST's cryptographic standards~\cite{pqc}.
To improve the input distribution $\mu_{\mathcal{X}}$, ~\cite{10.1145/3576915.3623076} applied a quite heavy preprocessing of LWE's initial inputs (which were uniformly distributed on ${\mathbb Z}_q^n$) that computed other inputs (with corresponding outputs) but with a lower collision entropy.

The code can be accessed on \href{https://github.com/k-nic/grad_info}{GitHub}, allowing for easy reproduction of our results.

\begin{figure}
\begin{minipage}[t]{.3\textwidth}
    \centering
    \includegraphics[width=.9\textwidth]{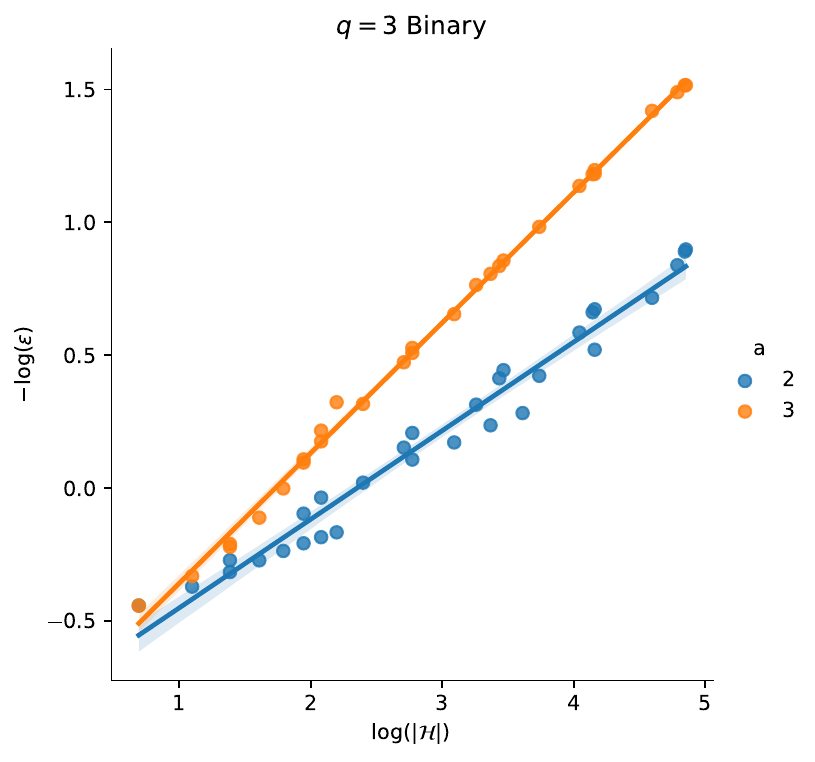}
\end{minipage}\hfill\begin{minipage}[t]{.3\textwidth}
    \centering
    \includegraphics[width=.9\textwidth]{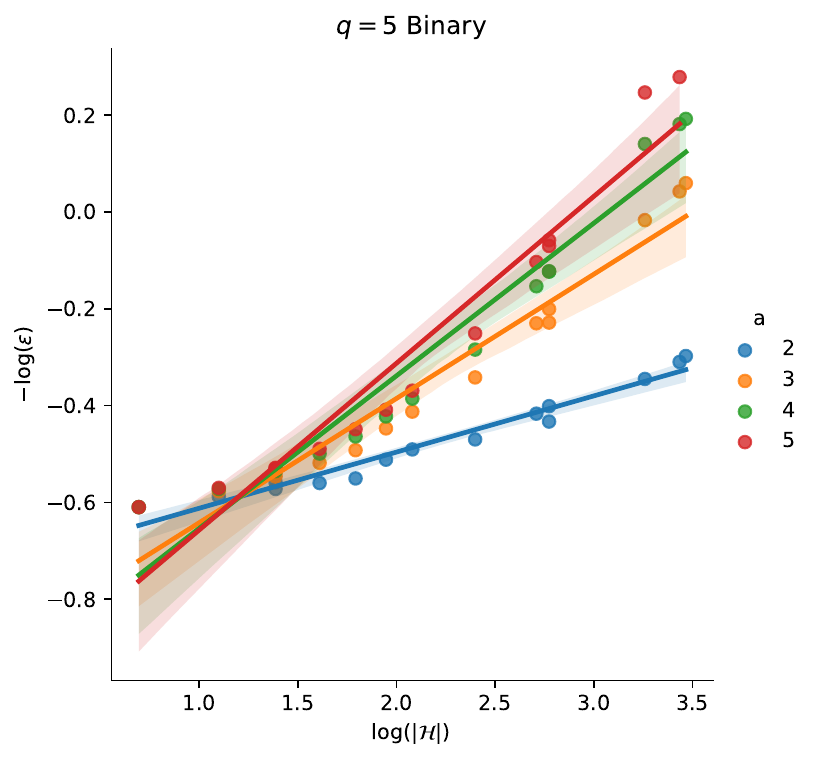}
\end{minipage}\hfill\begin{minipage}[t]{.3\textwidth}
    \centering
    \includegraphics[width=.9\textwidth]{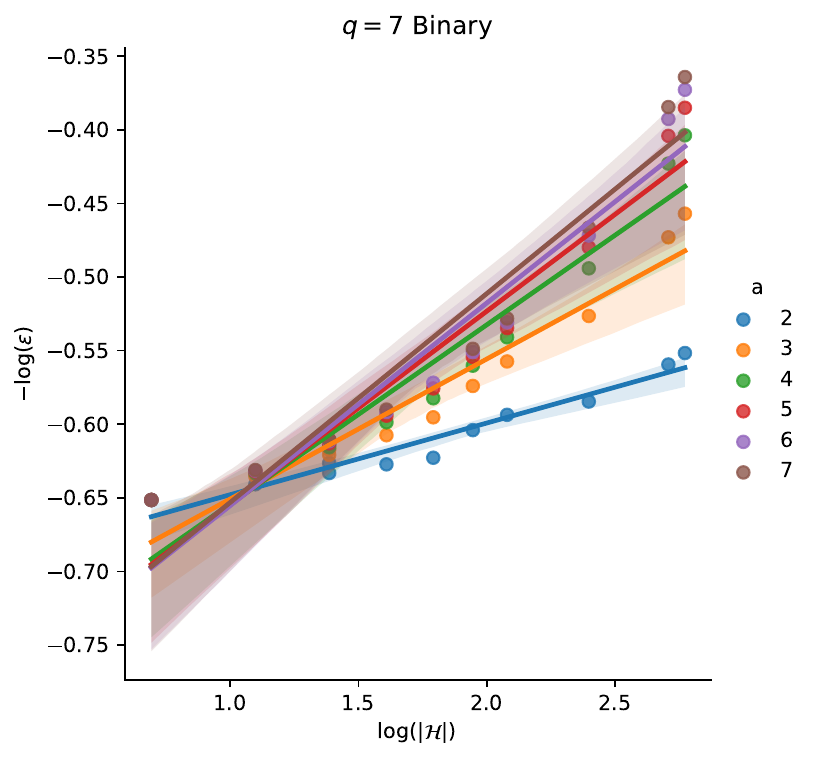}
\end{minipage}
\begin{minipage}[t]{.3\textwidth}
    \centering
    \includegraphics[width=.9\textwidth]{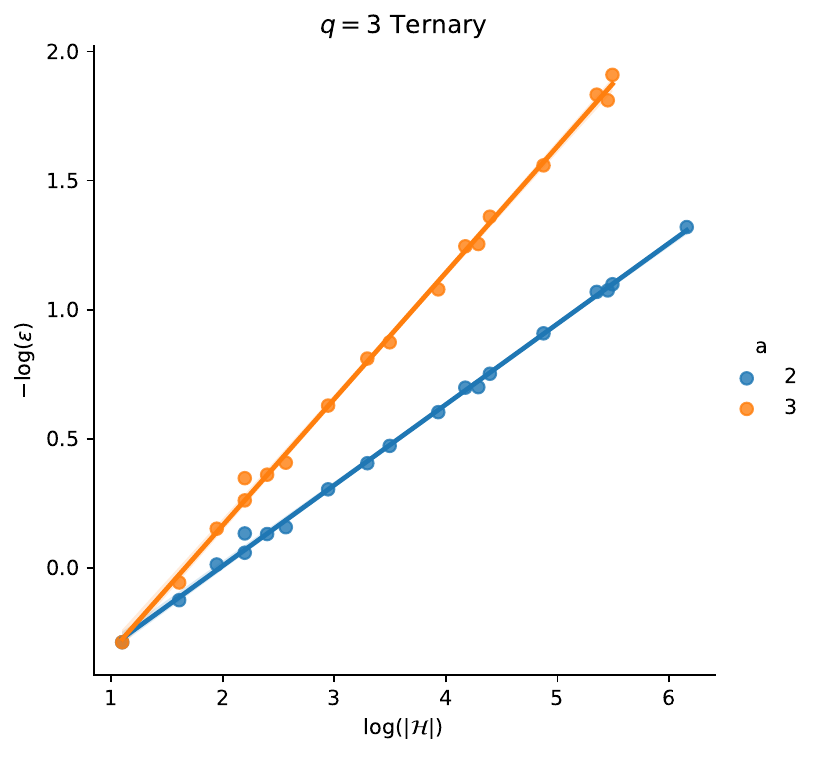}
\end{minipage}\hfill\begin{minipage}[t]{.3\textwidth}
    \centering
    \includegraphics[width=.9\textwidth]{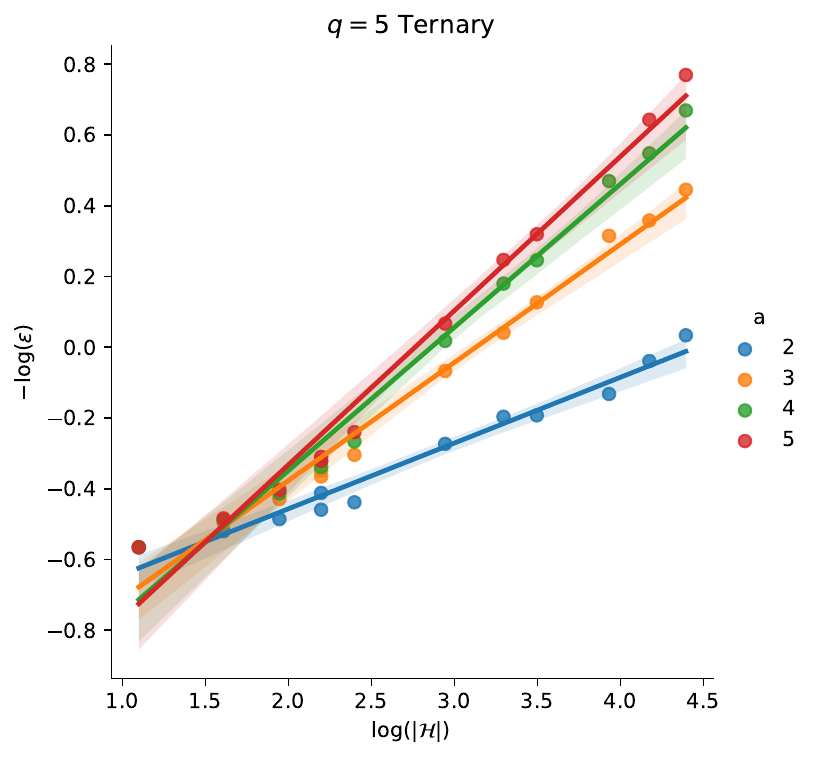}
\end{minipage}\hfill\begin{minipage}[t]{.3\textwidth}
    \centering
    \includegraphics[width=.9\textwidth]{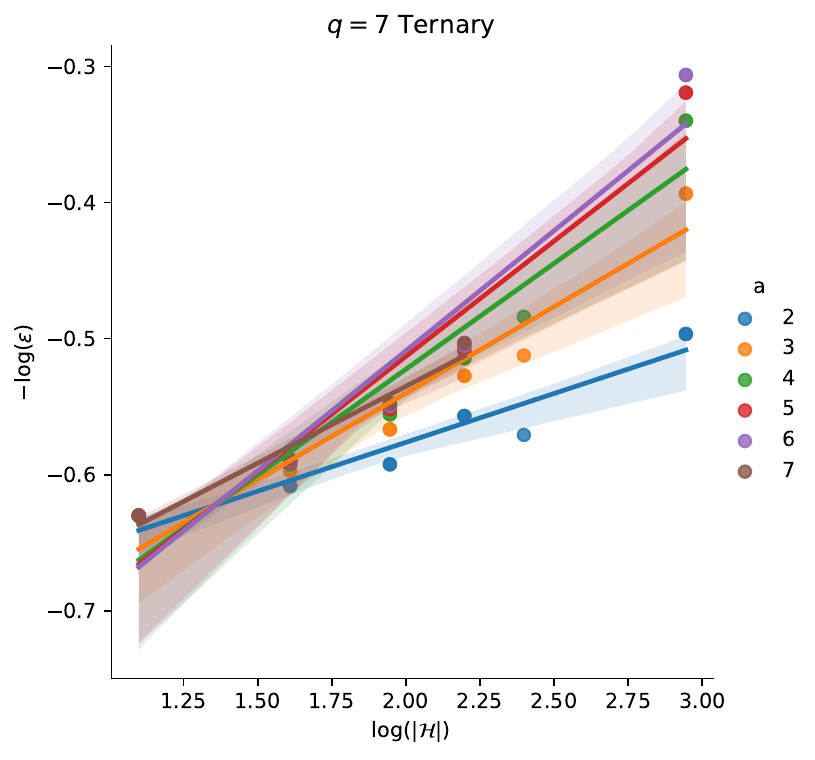}
\end{minipage}
\caption{Scatter and linear regression plots for ``$-\log(\varepsilon)$ vs $\log(|\mathcal{H}|)$'' for different $a=2,\cdots,q$.}\label{lrplots}    
\end{figure}

\section{Application of Theorem~\ref{aes-bound} to high-frequency functions}\label{waves}
We formulated Theorem~\ref{aes-bound} for probabilistic measures, as this general approach enables us to establish non-trivial results about the learnability of functions defined on continuous probability spaces.

Let us now consider $\mathcal{X} = {\mathbb R}^n$, $\mathcal{Y}={\mathbb R}$ and
\begin{equation*}
\begin{split}
\mathcal{H} = \{h:{\mathbb R}^n\to {\mathbb R} \mid h({\mathbf x}) = \psi({\mathbf a}^\top {\mathbf x}), {\mathbf a}\in {\mathbb R}^n\},
\end{split}
\end{equation*}
where $\psi: {\mathbb R}\to {\mathbb R}$ is 1-periodic, i.e. $\psi(x+1)=\psi(x)$, and measurable.  
We assume that the distribution $\chi$ over $\mathcal{H}$ is induced by sampling ${\mathbf A}\sim \mathcal{N}({\mathbf 0}, R^2 I_n)$ and setting $h({\mathbf x}) = \psi({\mathbf A}^\top {\mathbf x})$. The distribution $\mu_\mathcal{Y}$ is defined as the pushforward measure of the uniform distribution on $[0,1]$ with respect to $\psi$; that is, $\mu_\mathcal{Y} (S)= {\rm length}(\{x\in [0,1] \mid \psi(x)\in S\})$.

The following Theorem is a consequence of our main Theorem~\ref{aes-bound}. Its proof is based on estimating $\varepsilon_\mathcal{F}({\mathbf x}, {\mathbf y})$ for $\mathcal{F}_1 = L_\infty({\mathbb R})$ and $\mathcal{F}_2 = L_\infty({\mathbb R}^{2})$. In other words, by showing that high-frequency functions are almost pairwise independent if the parameter $R$ is large.

\begin{theorem}\label{n-more-2}
Let $n\geq 2$ and let $M_{\mathbf x} = \sup_{y\in {\mathbb R}} |r_{\mathbf w}({\mathbf x}, y)|$. If $\mu_{\mathcal X}$ is an absolutely continuous distribution, then
\begin{equation*}
\begin{split}
&{\rm Var}_{h\sim\chi}\big[\partial_{w_i} {\mathbb E}_{X\sim \mu_{\mathcal{X}}} [L(p({\mathbf w},X), h(X))]\big]\lesssim \\
&\frac{1}{R^2}\|\frac{\partial p({\mathbf w},\cdot)}{\partial w_i}\|^2_{\mu_{\mathcal{X}}}  {\mathbb E}_{({\mathbf x},{\mathbf y})\sim \mu_{\mathcal{X}}^2}\bigg[\frac{(\|{\mathbf x}\|\|{\mathbf y}\|+|{\mathbf x}^\top {\mathbf y}|)^2(\|{\mathbf x}\|+\|{\mathbf y}\|)^2}{(\|{\mathbf x}\|^2\|{\mathbf y}\|^2-({\mathbf x}^\top {\mathbf y})^2)^2}M_{{\mathbf x}}^2M_{{\mathbf y}}^2\bigg]^{\frac{1}{2}}.
\end{split}
\end{equation*}
\end{theorem}

\begin{remark} 
Theorem~\ref{n-more-2} states that the informativeness of the gradient decreases as $R$ grows, due to the increasing pairwise independence of the class $\mathcal{H}$. Assuming $M_{{\mathbf x}}$ and $\|\partial_{w_i}p({\mathbf w},\cdot)\|^2_{\mu_{\mathcal{X}}}$ are moderate (i.e., the loss function $L$ and the neural network $p$ are well-behaved), informativeness is determined by $R$ and the final term.
Note that the kernel $\frac{1}{(\|{\mathbf x}\|^2\|{\mathbf y}\|^2-({\mathbf x}^\top {\mathbf y})^2)^2}$ is singular for ${\mathbf x} = \lambda {\mathbf y}$. The term under the square root measures the regularity of the input distribution $\mu_\mathcal{X}$ and can be understood as a continuous analog of collision entropy (recall that collision entropy influences the gradient informativeness in the LWE context). So, in both discrete and continuous cases we see that the input distribution $\mu_\mathcal{X}$, if chosen properly, can potentially increase the RHS of our inequalities.
\end{remark}

\begin{remark} 
An analogous bound on the gradient's variance was proved in \cite{DBLP:journals/jmlr/Shamir18} for the specific case where $L(y,y') = (y-y')^2$ and ${\mathbf A}$ is sampled uniformly from $R\cdot {\mathbb S}^{n-1}$. This special case allows the use of the Fourier transform to bound the variance in terms of the decay rate of $\hat{\phi}$, where $\phi^2$ is the probability density function of $\mu_\mathcal{X}$. If $\hat{\phi}$ decays exponentially (as with a Gaussian distribution for $\mu_\mathcal{X}$), this bound becomes exponential in $R$ and significantly stronger than ours. In contrast, our bound holds for any loss function $L$, leading to a slower decay rate of $\frac{1}{R^2}$. It remains an open question whether this is simply the artifact of our analysis (which relies on pairwise independence) or if it stems from the specifics of the squared error loss and special assumptions on $\hat{\phi}$. 
\end{remark}

\begin{remark} \label{singular-kernel}
Let $\mu_{\mathcal{X}}$ be defined by some probability density function $f$ and $M_{{\mathbf x}}\lesssim 1$. Let us denote $g(r) = \sup_{{\mathbf x}: \|{\mathbf x}\|=r} f({\mathbf x})$.  For the expression $${\mathbb E}_{({\mathbf x},{\mathbf y})\sim \mu_{\mathcal{X}}^2}\bigg[\frac{(\|{\mathbf x}\|\|{\mathbf y}\|+|{\mathbf x}^\top {\mathbf y}|)^2(\|{\mathbf x}\|+\|{\mathbf y}\|)^2}{(\|{\mathbf x}\|^2\|{\mathbf y}\|^2-({\mathbf x}^\top {\mathbf y})^2)^2}\bigg]$$ to be bounded it is enough to have $n\geq 5$ and $\int_0^\infty r^{n-1}g(r)dr<\infty$, \\$\int_0^\infty r^{n-3}g(r)dr<\infty$. This is shown in the appendix.
\end{remark}

The case of $n=1$ is a special one and should be treated distinctly.

First, let us introduce the space of functions of bounded variation on $[0,1]$, denoted by ${\rm BV}([0,1])$, as a set of functions $f: [0,1]\to {\mathbb R}$  for which the following expression is bounded:
\begin{equation*}
\begin{split}
\|f\|_{{\rm BV}([0,1])}\triangleq\sup_{\substack{N\in {\mathbb N} \\ 0= x_0\leq ...\leq x_N= 1}}\sum_{i=1}^N |f(x_i)-f(x_{i-1})|.
\end{split}
\end{equation*} 
It is well-known that such functions are always measurable~\cite{zbMATH01448982}. 
\begin{theorem}\label{case-n-1}
Let $n=1$, $\psi|_{[0,1]}\in {\rm BV}([0,1])$ and $|\frac{\partial L(p({\mathbf w},x), y)}{\partial p}|\lesssim  1$, $|\frac{\partial^2 L(p({\mathbf w},x), y)}{\partial p \partial y}|\\ \lesssim  1$, $|\frac{\partial^3 L(p({\mathbf w},x), y)}{\partial p \partial y^2}|\lesssim  1$. Also, let $\mu_X$ be an absolutely continuous distribution. Then,  
\begin{equation*}
\begin{split}
&{\rm Var}_{h\sim\chi}\big[\partial_{w_i} {\mathbb E}_{X\sim \mu_{\mathcal{X}}} [L(p({\mathbf w},X), h(X))]\big]\lesssim \\
& \|\frac{\partial p({\mathbf w},\cdot)}{\partial w_i}\|^2_{\mu_{\mathcal{X}}}\max(\|\psi\|_{{\rm BV}([0,1])}, \|\psi\|^2_{{\rm BV}([0,1])})\times \\
&\min_{H\in {\mathbb N}}\bigg(\frac{1}{H^2}+
\log^2 (H+1) \sum_{\substack{a,b\in [-H,H]\cap {\mathbb Z},\\ (a,b)\ne (0,0)}} \frac{{\mathbb E}_{(X,Y)\sim \mu_{\mathcal{X}}^2}[e^{-2\pi^2R^2(aX+bY)^2}]}{\max(|a|,1)\max(|b|,1)}\bigg)^{\frac{1}{2}}.
\end{split}
\end{equation*}
\end{theorem}

\begin{example}\label{example-case-n-1}
Let $\mu_\mathcal{X}$ be $U([0,1])$. Then,
\begin{equation*}
\begin{split}
&{\mathbb E}_{(X,Y)\sim \mu_{\mathcal{X}}^2}[e^{-2\pi^2R^2(aX+bY)^2}] = \int_{[0,1]^2}e^{-2\pi^2R^2(ax+by)^2}dxdy\leq \\
&\sqrt{2} 
\int_{0}^{\sqrt{2}} e^{-2\pi^2R^2(a^2+b^2)t^2}dt= \frac{1}{\pi R\sqrt{a^2+b^2}} 
\int_{0}^{2\pi R\sqrt{ a^2+b^2}} e^{-t^2}dt\lesssim \frac{1}{R}.
\end{split}
\end{equation*}
Therefore,
\begin{equation*}
\begin{split}
&\sum_{\substack{a,b\in [-H,H]\cap {\mathbb Z},\\ (a,b)\ne (0,0)}} \frac{{\mathbb E}_{(X,Y)\sim \mu_{\mathcal{X}}^2}[e^{-2\pi^2R^2(aX+bY)^2}]}{\max(|a|,1)\max(|b|,1)}\lesssim \frac{\log^2 (H+1)}{R}.
\end{split}
\end{equation*}
Setting $H=\lfloor R\rfloor$ gives that the minimum over $H\in {\mathbb N}$ in Theorem~\ref{case-n-1} is asymptotically bounded by $\frac{\log^4 (R+1)}{R+1}$. Thus,
\begin{equation*}
\begin{split}
&{\rm Var}_{h\sim\chi}\big[\partial_{w_i} {\mathbb E}_{X\sim \mu_{\mathcal{X}}} [L(p({\mathbf w},X), h(X))]\big]\lesssim \frac{\log^2(R+1)}{\sqrt{R+1}} \\  
&
\times \|\frac{\partial p({\mathbf w},\cdot)}{\partial w_i}\|^2_{\mu_{\mathcal{X}}}\max(\|\psi\|_{{\rm BV}([0,1])}, \|\psi\|^2_{{\rm BV}([0,1])}).
\end{split}
\end{equation*}

\end{example}

\section{Nature of the barren plateau phenomenon: example}
Let us demonstrate the typical landscape of the objective function for the case of learning high-frequency functions with \( n=1 \), specifically for the hypothesis set defined by
\begin{equation*}
\begin{split}
\mathcal{H} = \{h(x) = \psi (wx)\mid w>0\}.
\end{split}
\end{equation*}
where $\psi: {\mathbb R}\to {\mathbb R}$ is 1-periodic. The objective function is given by
\[
\mathbb{E}_{X \sim \mu_{\mathcal{X}}}\Big[L\big(p(\mathbf{w}, X), h(X)\big)\Big],
\]
where $\mu_{\mathcal{X}}$ is a uniform distribution over $[0,1]$.

If $\psi$ is regular enough, one can represent it as the Fourier series
\begin{equation*}
\begin{split}
\psi(x) = \sum_{i=-\infty}^{+\infty} a_i e^{{\rm i}2\pi ix},
\end{split}
\end{equation*}
where $a_i = \int_{0}^1 \psi(x)e^{-{\rm i}2\pi ix}dx$. Therefore, the target function has the following form
\begin{equation*}
\begin{split}
h(x) = \psi(wx)=\sum_{i=-\infty}^{+\infty} a_i e^{{\rm i}2\pi iwx}.
\end{split}
\end{equation*}
Note that the real part of the Fourier transform of $h$ has its local extrema at points $2\pi i w, i\in {\mathbb Z}$, due to
\begin{equation*}
\begin{split}
\widehat{h}(\omega) = \int_{-\infty}^{+\infty}h(x)e^{-{\rm i}\omega x}dx = 2\pi\sum_{i=-\infty}^{+\infty} a_i \delta(\omega-2\pi i w),
\end{split}
\end{equation*}
where $\delta$ is Dirac's delta function. 
Therefore, finding the local extrema of $\widehat{h}$ is equivalent to identifying the frequency $w$, which in turn is equivalent to successfully learning the target function $h$. Motivated by this, we define a loss function $L$ as $L(y,y') = y y'$ and consider a neural network of the form $p(\omega,x) = \cos(\omega x)$ or $p(\omega,x) = -\cos(\omega x)$, where $\omega\in {\mathbb R}$ is the network’s only parameter.
With this setup, our optimization task becomes
\begin{equation*}
\begin{split}
C_h(\omega)= \int_0^1 h(x)\cos(\omega x) dx = {\rm Re}\int_0^1 h(x)e^{-{\rm i}\omega x}dx 
\to 
\mathop{\min}\limits_{\omega} ( {\rm or\,\,}\mathop{\max} \limits_{\omega}).
\end{split}
\end{equation*}
The latter objective represents the smoothed Fourier transform of $h$, i.e.
\begin{equation*}
\begin{split}
C_h(\omega)  =  \operatorname{Re}\int_0^1 h(x)e^{-{\rm i}\omega x}dx =  \operatorname{Re}\int_0^1 \sum_{i=-\infty}^{+\infty} a_i e^{{\rm i}2\pi iwx}e^{-{\rm i}\omega x}dx = \\
\operatorname{Re} \sum_{i=-\infty}^{+\infty} a_i k(\omega-2\pi iw),
\end{split}
\end{equation*}
where $k(\omega) = \int_0^1 e^{-{\rm i}\omega x} dx = \frac{{\rm i}}{\omega}(e^{-{\rm i}\omega } -1)$. Note that $|k(\omega)| = \frac{2|\sin(\frac{\omega}{2})|}{|\omega|}$. In other words,
\begin{equation*}
\begin{split}
C_h(\omega)  =  \operatorname{Re} (\widehat{h}\ast \frac{k}{2\pi})(\omega),
\end{split}
\end{equation*}
where $\ast$ denotes the convolution.

\begin{figure}
\begin{minipage}[t]{0.5\textwidth}
    \centering
    \includegraphics[width=1.0\textwidth]{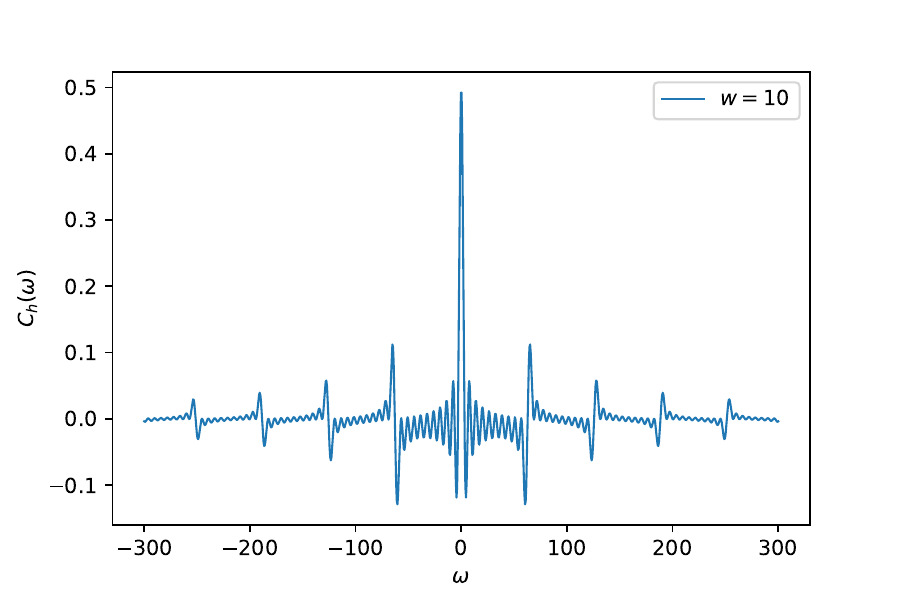}
\end{minipage} \hfill \begin{minipage}[t]{0.5\textwidth}
    \centering
    \includegraphics[width=1.0\textwidth]{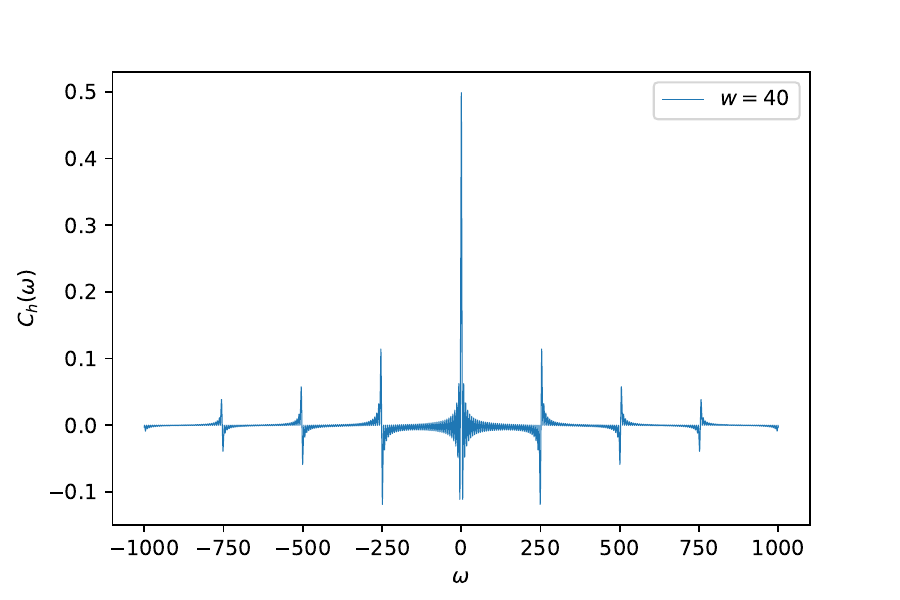}
\end{minipage}
\caption{The objective function $C_h(\omega)$ for $\psi(x) = \{x\}$, where $\{x\}$ is a fractional part of $x$, and for frequencies $w=10$ and $w=40$.}\label{obj}    
\end{figure}

When the frequency \( w \) is large, the distance $2\pi w$ between the local extrema of $\widehat{h}(\omega)$ becomes large. In this case, applying the convolution minimally shifts the positions of these local extrema, meaning they remain close to the points $ 2\pi w i$, for $i \in \mathbb{Z}$. Additionally, the function $\operatorname{Re} \left( \widehat{h} \ast \frac{k}{2\pi} \right)(\omega)$ stays very close to zero between these extrema, similar to the behavior of the non-smoothed function $\widehat{h}$. 

In other words, $C_h(\omega)$ remains small within intervals $[2\pi w i+\varepsilon, 2\pi w (i+1)-\varepsilon]$ for a small $\varepsilon>0$ (as illustrated in Figure~\ref{obj}, the function can exhibit highly oscillatory behavior). In these regions, the gradient of $C_h(\omega)$ is not informative, providing useful signals only when $\omega$ is already near a local extremum. This behavior explains the so-called ``grokking'' effect observed during training at moderate frequencies $w$. Experimentally, we observed how, after many epochs of stagnation, the algorithm suddenly finds a local minimum within a few epochs.

\begin{remark} Formally, the fact that the gradient of $C_h(\omega)$ carries little useful information for high frequencies follows from Theorem~\ref{case-n-1} and Example~\ref{example-case-n-1}. In our analysis, we quantify the informativeness of the gradient by measuring its variance with respect to a randomly chosen target function.

An alternative measure of gradient informativeness is the so-called “signal-to-noise ratio” (SNR). For the mean squared error loss, the SNR is defined as the ratio between  
$${\rm Sig}_\mathbf{w} = \Big\|\mathbb{E}_{X \sim \mu_{\mathcal{X}}}\big[h(X)\frac{\partial p(\mathbf{w}, X)}{\partial \mathbf{w}}\big]\Big\|^2$$  
and  
$${\rm Noi}_\mathbf{w} = \mathbb{E}_{X \sim \mu_{\mathcal{X}}}\Big\|h(X)\frac{\partial p(\mathbf{w}, X)}{\partial \mathbf{w}} - \mathbb{E}_{X' \sim \mu_{\mathcal{X}}}\big[h(X')\frac{\partial p(\mathbf{w}, X')}{\partial \mathbf{w}}\big]\Big\|^2.$$  
A small SNR implies unstable gradient behavior across batches, leading to slower convergence during optimization.

It is important to note that, in the SNR definition, randomness comes from sampling the input $X$, rather than from sampling the target function, as in our approach. This distinction is significant. In our framework, the difficulty of learning arises from the gradient having negligibly small variance with respect to the target function. In contrast, under the SNR perspective, optimization becomes challenging due to high variance with respect to $X$.

Finally, a low variance with respect to the target function and a large variance (relative to ${\rm Sig}_\mathbf{w}$) with respect to the input data are not mutually exclusive. In fact, as shown in Section 3.2 of the full version of~\cite{DBLP:conf/icml/Shalev-ShwartzS17}, in one case of the barren plateau phenomenon, when the variance with respect to targets is negligibly small, experimentally computed SNR quantity is also small.
\end{remark}

\begin{remark} 
The problem of training high-frequency functions of a single variable has been analyzed in~\cite{Takhanov2025}. According to Remark 4 of~\cite{Takhanov2025}, if $\psi: {\mathbb R} \to \{-1, 1\}$ is a balanced 1-periodic function (i.e., $\psi(x+1) = \psi(x)$ and $\int_0^1 \psi(x) dx = 0$), and $A$ is a large natural number, then any statistical query (SQ) algorithm that learns the concept class  
$$  
\{\psi(ax) \mid a \in \{0, 1, \dots, A-1\}\}  
$$  
requires at least $C A^{1/6}$ queries. Here, $C$ is a constant that may depend on $\psi$ and can include logarithmic factors in $A$. Since SQ algorithms include soft-margin SVMs for binary classification problems (possibly in a transformed feature space)~\cite{DBLP:conf/soda/FeldmanGV17}, this result implies that high-frequency targets are also difficult to learn with SVMs.

This observation naturally leads to the conjecture that training high-frequency functions is similarly hard for kernel ridge regression (KRR). Recall that, in the wide-width limit, appropriately initialized neural networks trained via gradient descent are equivalent to KRR~\cite{Jacot}. As a result, it is reasonable to conjecture that high-frequency targets are also challenging for over-parameterized neural networks.
Experimental results support this conjecture (see Figure~\ref{over-p} in the Appendix).
\end{remark}

\section{Applications}
This study informs practitioners about the limitations of neural networks in learning mappings from almost pairwise independent classes of functions, even when deep learning architectures are varied. This insight can be potentially applied to design of secure encryption schemes resilient to modern machine learning-based cryptanalysis~\cite{Vaudenay2003}. The findings are particularly relevant for post-quantum cryptographic protocols, like those relying on LWE. Furthermore, the theoretical framework that we use can guide the evaluation and benchmarking of new cryptographic algorithms, helping to identify whether a given primitive is inherently robust against gradient-based learning~\cite{CiC-1-3-32}.  It also provides a theoretical justification for why side-channel attacks using neural networks might require preprocessing or non-standard data distributions to succeed, thus shaping the direction of future cybersecurity research.

\section{Conclusions and open problems}
Training parity functions or high-frequency functions using specific losses (such as mean squared error loss) by gradient-based methods is hard, due to a low informativeness of the gradient. 
We return to the latter fact and enhance it by the following observation: any almost pairwise independent class of target functions, trained using any regular neural network architecture and any regular loss function, and by any gradient-based method, encounters the same issue of a low informativeness of the gradient, unless the input distribution is carefully tuned. The last condition is crucial --- recent attacks on LWE demonstrate that specific input distributions improve the quality of the gradient. 

As this attack demonstrated, even if we lack direct access to input-output pairs from a ``favorable'' input distribution, such pairs can often be derived from those sampled from a ``less favorable'' distribution (for example, one generated from a uniform distribution over inputs). This motivates an interesting research direction: identifying conditions under which input-output pairs from a ``favorable'' distribution can be efficiently computed from pairs sampled from a less informative distribution, such as a uniform one.

\section*{Acknowledgments}
This research has been funded by the Science Committee of the Ministry of Science and Higher Education of the Republic of Kazakhstan (Grant No. AP27510283), PI R. Takhanov.




\appendix
\section{Proof of Theorem~\ref{aes-bound}}
Since $p: O\times \mathcal{X}\to {\mathbb R}$ and ${\rm id}_{\mathcal{Y}}:\mathcal{Y}\to \mathcal{Y}, {\rm id}_{\mathcal{Y}}(y)=y$ are measurable functions, $({\mathbf w},x,y)\to (p({\mathbf w},x)+c, y)$ is a measurable function between $O\times \mathcal{X}\times \mathcal{Y}$ and ${\mathbb R}\times \mathcal{Y}$ for any $c\in {\mathbb R}$. Therefore, the composition $({\mathbf w},x,y)\to (p({\mathbf w},x)+c, y)\to L(p({\mathbf w},x)+c, y)$ is measurable. Thus, $n(L(p({\mathbf w},x)+\frac{1}{n}, y)-L(p({\mathbf w},x), y))$ is measurable on $O\times \mathcal{X}\times \mathcal{Y}$ and the limit $k({\mathbf w},x,y) = \limsup_{n\to +\infty}n(L(p({\mathbf w},x)+\frac{1}{n}, y)-L(p({\mathbf w},x), y))$ is an $\overline{{\mathbb R}}$-valued measurable function on $O\times \mathcal{X}\times \mathcal{Y}$. Let us assume that if $\lim_{\Delta p\to 0}\frac{L(p({\mathbf w},x)+\Delta p, y)-L(p({\mathbf w},x), y)}{\Delta p}$ is undefined, then $\frac{\partial L(p({\mathbf w},x), y)}{\partial p}$ is defined to be equal to $k({\mathbf w},x,y)$. Since we always have $k({\mathbf w},x,y) = \frac{\partial L(p({\mathbf w},x), y)}{\partial p}$, the function $\frac{\partial L(p({\mathbf w},x), y)}{\partial p}$ is measurable on $O\times \mathcal{X}\times \mathcal{Y}$. 

From the a.e. boundedness of $\frac{\partial L(p({\mathbf w},x), y)}{\partial p}$ and the Fubini-Tonelli theorem we conclude that the function is integrable w.r.t. to $\mu_{\mathcal{Y}}$ a.e. over $({\mathbf w},x)\in O\times \mathcal{X}$ and ${\mathbb E}_{Y\sim \mu_{\mathcal{Y}}}[\frac{\partial L(p({\mathbf w},x), Y)}{\partial p}]$ is measurable. Let us denote
\begin{equation*}
\begin{split}
r_{\mathbf w}(x,y) = \frac{\partial L(p({\mathbf w},x), y)}{\partial p}-{\mathbb E}_{Y\sim \mu_{\mathcal{Y}}}[\frac{\partial L(p({\mathbf w},x), Y)}{\partial p}].\\
\end{split}
\end{equation*}
Thus, $r_{\mathbf w}(x,y)$ is measurable on $O\times \mathcal{X}\times \mathcal{Y}$ and a.e. bounded. 

Since $p: O\times \mathcal{X}\to {\mathbb R}$ and $(x,h)\to h(x)$ are measurable functions, $({\mathbf w},x,h)\to (p({\mathbf w},x)+c, h(x))$ is a measurable function between $O\times \mathcal{X}\times \mathcal{H}$ and ${\mathbb R}\times \mathcal{Y}$ for any $c\in {\mathbb R}$. Therefore, 
the composition $({\mathbf w}, x,h)\to (p({\mathbf w},x)+c, h(x))\to L(p({\mathbf w},x)+c, h(x))$  is measurable and bounded on $O\times \mathcal{X}\times \mathcal{H}$.
Analogously to the case of  $r_{\mathbf w}(x,y)$, from the measurability of $L(p({\mathbf w},x)+c, h(x))$ on $O\times \mathcal{X}\times \mathcal{H}$, we conclude that $\frac{\partial L(p({\mathbf w},x), h(x))}{\partial p}$ and $r_{\mathbf w}(x,h(x))$ are measurable and a.e. bounded on $O\times \mathcal{X}\times \mathcal{H}$.

For any $h$, $\frac{\partial L(p({\mathbf w},x), h(x))}{\partial p}$ is measurable on $O\times \mathcal{X}$ and 
from the Fubini-Tonelli theorem we obtain that a.e. over $h\sim \chi$, $\frac{\partial L(p({\mathbf w},x), h(x))}{\partial p}$ is a.e. bounded on $O\times \mathcal{X}$.
Thus,
\begin{equation*}
\begin{split}
\frac{\partial L(p({\mathbf w},x), h(x))}{\partial p}\frac{\partial p({\mathbf w},x)}{\partial w_i}
\end{split}
\end{equation*}
is measurable, a.e. bounded on $O\times \mathcal{X}$ and can differ from the derivative $\partial_{w_i} L(p({\mathbf w},x), h(x))$ only in some subset of a null-set. I.e. $\frac{\partial L(p({\mathbf w},x), h(x))}{\partial p}\frac{\partial p({\mathbf w},x)}{\partial w_i}=\partial_{w_i} L(p({\mathbf w},x), h(x))$ a.e. over $({\mathbf w},x)\in \overline{O\times \mathcal{X}}$. Using the Leibniz integral rule in measure theory we obtain
\begin{equation*}
\begin{split}
&\partial_{w_i} {\mathbb E}_{X\sim \mu_{\mathcal{X}}}[L(p({\mathbf w},X), h(X))]  = 
\partial_{w_i} {\mathbb E}_{X\sim \overline{\mu_{\mathcal{X}}}}[L(p({\mathbf w},X), h(X))] = \\
&{\mathbb E}_{X\sim \overline{\mu_{\mathcal{X}}}}[\partial_{w_i} L(p({\mathbf w},X), h(X))] = 
{\mathbb E}_{X\sim \mu_{\mathcal{X}}}[\frac{\partial L(p({\mathbf w},X), h(X))}{\partial p}\frac{\partial p({\mathbf w},X)}{\partial w_i}].
\end{split}
\end{equation*}
Since $\frac{\partial L(p({\mathbf w},x), h(x))}{\partial p}$ is measurable on $O\times \mathcal{X}\times \mathcal{H}$ and is a.e. bounded, the Fubini-Tonelli theorem  gives us that the previous expression is integrable w.r.t. $\chi$.
Thus, we have
\begin{equation*}
\begin{split}
&{\rm Var}_{h\sim \chi}\big[\partial_{w_i} {\mathbb E}_{X\sim \mu_{\mathcal{X}}}L(p({\mathbf w},X), h(X))\big] = \\
&{\rm Var}_{h\sim \chi}\big[ {\mathbb E}_{X\sim \mu_{\mathcal{X}}}[\frac{\partial L(p({\mathbf w},X), h(X))}{\partial p}\frac{\partial p({\mathbf w},X)}{\partial w_i}]\big]=\\
&{\rm Var}_{h\sim \chi}\big[ {\mathbb E}_{X\sim \mu_{\mathcal{X}}}[r_{\mathbf w}(X,h(X))\frac{\partial p({\mathbf w},X)}{\partial w_i}]\big]\leq 
{\mathbb E}_{h\sim \chi}\big[ \langle \frac{\partial p({\mathbf w},x)}{\partial w_i}, r_{\mathbf w}(x, h(x))\rangle^2_{\mu_{\mathcal{X}}}\big].
\end{split}
\end{equation*}
\begin{lemma}\label{through-F} Let $f: \mathcal{X}\times \mathcal{H}\to \overline{{\mathbb R}}$ be measurable and a.e. bounded on $\mathcal{X}\times \mathcal{H}$ and let $f_h(x) = f(x, h)$. 
Then, for any $g\in L_2(\mathcal{X}, \mu_{\mathcal{X}})$ we have
$$
{\mathbb E}_{h\sim \chi}[\langle f_h, g\rangle^2_{\mu_{\mathcal{X}}}]\le \|g\|^2_{\mu_{\mathcal{X}}}\sqrt{{\mathbb E}_{(h,h')\sim \chi^2}[\langle f_h,f_{h'}\rangle^2_{\mu_{\mathcal{X}}}]}.
$$
\end{lemma}
\begin{proof}
Let us introduce a linear operator $T: L_2(\mathcal{X}, \mu_{\mathcal{X}})\to L_2(\mathcal{H}, \chi)$ by $T[g](h) = \langle f_h, g\rangle_{\mu_{\mathcal{X}}}$. By construction, $$\|T\| = \sup_{g\in L_2(\mathcal{X}, \mu_{\mathcal{X}}): \|g\|_{\mu_{\mathcal{X}}} \leq 1}\sqrt{{\mathbb E}_{h\sim \chi}[|T[g](h)|^2]}$$ and 
\begin{equation*}
\begin{split}
{\mathbb E}_{h\sim \chi}[\langle f_h, g\rangle^2_{\mu_{\mathcal{X}}}]\le \|T\| ^2 \|g\|^2_{\mu_{\mathcal{X}}}.
\end{split}
\end{equation*}
 Let us calculate the adjoint of $T$. For any $\phi\in L_2(\mathcal{H}, \chi)$, using the Fubini-Tonelli theorem we have 
\begin{equation*}
\begin{split}
&\langle \phi, T[g]\rangle_{\chi} = \int_{\mathcal{H}}\phi(h)\langle f_h, g\rangle_{\mu_{\mathcal{X}}}d\chi(h) = \int_{\mathcal{H}}\phi(h)\int_{\mathcal{X}} f(x, h) g(x)d\mu_{\mathcal{X}}(x)d\chi(h)=\\
&\int_{\mathcal{X}}g(x) \int_{\mathcal{H}}\phi(h) f(x, h) d\chi(h)d\mu_{\mathcal{X}}(x) = 
\langle \int_{\mathcal{H}}\phi(h) f(x, h) d\chi(h), g(x)\rangle_{\mu_{\mathcal{X}}}=\\
&\langle T^\dag[\phi], g\rangle_{\mu_{\mathcal{X}}}.
\end{split}
\end{equation*}
Therefore, $T^\dag[\phi](x) = \int_{\mathcal{H}}\phi(h) f(x, h) d\chi(h)$ and 
\begin{equation*}
\begin{split}
&TT^\dag[\phi](h) = \int_{\mathcal{X}} f(x, h)  \int_{\mathcal{H}}\phi(h') f(x, h') d\chi(h')d\mu_{\mathcal{X}}(x)=\\
&\int_{\mathcal{H}} \int_{\mathcal{X}} f(x, h)  \phi(h') f(x, h') d\mu_{\mathcal{X}}(x)d\chi(h') = \int_{\mathcal{H}} \langle f_h, f_{h'}\rangle_{\mu_{\mathcal{X}}}\phi(h')d\chi(h').
\end{split}
\end{equation*}
Since $\langle f_h, f_{h'}\rangle_{\mu_{\mathcal{X}}}$ is bounded, we have $\langle f_h, f_{h'}\rangle_{\mu_{\mathcal{X}}}\in L_2(\mathcal{H},\chi\times \chi)$. 
 Therefore, $TT^\dag$ is the Hilbert-Schmidt operator $g\mathop\to\limits^{TT^\dag} \int_{\mathcal{H}^2}\langle f_h, f_{h'}\rangle_{\mu_{\mathcal{X}}} g(h')d\chi(h')$ on $L_2(\mathcal{H}, \chi)$ with the Hilbert-Schmidt norm 
 $$ \sqrt{\int_{\mathcal{H}^2} \langle f_h, f_{h'}\rangle^2_{\mu_{\mathcal{X}}}d\chi^2(h,h')}.$$ 
Since the Hilbert-Schmidt norm is an upper bound on the operator norm (for details, see the proposition 4.7 from~\cite{Conway2007}), we conclude
\begin{equation*}
\begin{split}
\|T\|^2= \|T T^\dag\| \leq \sqrt{{\mathbb E}_{(h,h')\sim \chi^2}[\langle f_h, f_{h'}\rangle^2_{\mu_{\mathcal{X}}}}].
\end{split}
\end{equation*}
\end{proof}

After setting $f(x,h)=r_{\mathbf w}(x, h(x))$ and using Lemma~\ref{through-F}, we obtain
\begin{equation}\label{BB-bound}
\begin{split}
&{\mathbb E}_{h\sim \chi}[\langle r_{\mathbf w}(x, h(x)), \frac{\partial p({\mathbf w},x)}{\partial w_i}\rangle^2_{\mu_{\mathcal{X}}}]\leq  \\
&\|\frac{\partial p({\mathbf w},\cdot)}{\partial w_i}\|^2_{\mu_{\mathcal{X}}}({\mathbb E}_{(h_1, h_2)\sim \chi^2}[\langle r_{\mathbf w}(x, h_1(x)), r_{\mathbf w}(x, h_2(x))\rangle^2_{\mu_{\mathcal{X}}}]\big)^{\frac{1}{2}}.
\end{split}
\end{equation}
The primary tool we use to bound the expectation in RHS of~\eqref{BB-bound} is the following lemma, with its proof provided in the next subsection.
\begin{lemma}\label{aes-rough} Let the function $g: \mathcal{X}\times \mathcal{Y}\to \overline{{\mathbb R}}$ be measurable and a.e. bounded on $\mathcal{X}\times \mathcal{Y}$. Also, suppose we have ${\mathbb E}_{Y\sim \mu_{\mathcal{Y}}}[g(x, Y)]=0$ for any $x\in \mathcal{X}$ and $f_{h}(x)=g(x, h(x))$ is measurable and a.e. bounded on $\mathcal{X}\times\mathcal{H}$. Then, we have
\begin{equation}
\begin{split}
&\sqrt{{\mathbb E}_{(h_1, h_2)\sim \chi^2}[\langle f_{h_1}, f_{h_2}\rangle^2_{\mu_{\mathcal{X}}}]}\leq 
\sqrt{{\mathbb E}_{(X,Y)\sim \mu_{\mathcal{X}}^2}\big[ \varepsilon(X,Y)^2\|\psi_{X,Y}\|_{\mathcal{F}_2}^2\big]}+ \\
&\sqrt{{\mathbb P}_{(X,Y)\sim \mu_{\mathcal{X}}^2}[X=Y]{\mathbb E}_{(X,Y)\sim \mu_{\mathcal{X}}^2}[(\|\psi_{X}\|_{\mathcal{F}_1}\varepsilon_\mathcal{F}(X,X)+D_X)^2\mid X=Y]}.
\end{split}
\end{equation}
where $D_x = {\rm Var}_{Y\sim \mu_{\mathcal{Y}}}(g(x, Y))$, $\psi_{x,x'}(y,y')=g(x, y)g(x', y')$ and $\psi_{x}(y)=g(x, y)^2$.
\end{lemma}
From Lemma~\ref{aes-rough}, after setting $g(x,y) = r_{\mathbf w}(x,y)$, $f_h(x) = g(x,h(x))$, $\phi_{x}(y)=r_{\mathbf w}(x,y)^2$ and $\phi_{x,x'}(y,y')=r_{\mathbf w}(x,y) r_{\mathbf w}(x',y')$, we obtain
\begin{equation*}
\begin{split}
&\big({\mathbb E}_{(h_1, h_2)\sim \chi^2}[\langle r_{\mathbf w}(x, h_1(x)), r_{\mathbf w}(x, h_2(x))\rangle^2_{\mu_{\mathcal{X}}}]\big)^{\frac{1}{2}}\leq \\
 &\sqrt{{\mathbb E}_{(X,Y)\sim \mu_{\mathcal{X}}^2}\big[ \varepsilon(X,Y)^2\|\phi_{X,Y}\|_{\mathcal{F}_2}^2\big]}+\\
&\sqrt{{\mathbb P}_{(X,Y)\sim \mu_{\mathcal{X}}^2}[X=Y]{\mathbb E}_{(X,Y)\sim \mu_{\mathcal{X}}^2}[(\|\phi_{X}\|_{\mathcal{F}_1}\varepsilon_\mathcal{F}(X,X)+D_X)^2\mid X=Y]}.
\end{split}
\end{equation*}
By substituting the latter inequality into the bound~\eqref{BB-bound}, we arrive at the following result:
\begin{equation*}
\begin{split}
&{\rm Var}_{h\sim \chi}\big[\partial_{w_i} {\mathbb E}_{X\sim \mu}L(p({\mathbf w},X), h(X))\big] \leq \\
&\|\frac{\partial p({\mathbf w},x)}{\partial w_i}\|^2_{\mu_{\mathcal{X}}} 
\big(
  \sqrt{{\mathbb E}_{(X,Y)\sim \mu_{\mathcal{X}}^2}\big[ \varepsilon(X,Y)^2\|\phi_{X,Y}\|_{\mathcal{F}}^2\big]}+\\
&\sqrt{{\mathbb P}_{(X,Y)\sim \mu_{\mathcal{X}}^2}[X=Y]{\mathbb E}_{(X,Y)\sim \mu_{\mathcal{X}}^2}[(\|\phi_{X}\|_{\mathcal{F}_1}\varepsilon_\mathcal{F}(X,X)+D_X)^2\mid X=Y]} \big).
\end{split}
\end{equation*}
This completes the proof.

\subsection{Proof of Lemma~\ref{aes-rough}}

To complete the proof of Theorem~\ref{aes-bound}, it is essential to establish Lemma~\ref{aes-rough}. The next lemma is central to achieving this. 
Let $g: \mathcal{X} \times \mathcal{Y} \to \mathbb{R}$ be the function satisfying the lemma's conditions, and define $f_h(x) = g(x, h(x))$.

\begin{lemma}\label{inversion} Let $F(x,y) =\int_{\mathcal{H}} f_{h}(x)f_{h}(y)d\chi(h)$.
Then,
\begin{equation*}
\begin{split}
{\mathbb E}_{(h_1, h_2)\sim \chi^2} [\langle f_{h_1}, f_{h_2}\rangle^2_{\mu_{\mathcal{X}}}] =
\int_{\mathcal{X}\times \mathcal{X}}F(x,y)^2 d\mu_{\mathcal{X}}^2(x,y).
\end{split}
\end{equation*}
\end{lemma}
\begin{proof}
A direct calculation using Fubini's theorem gives us
\begin{equation*}
\begin{split}
&{\mathbb E}_{(h_1, h_2)\sim \chi^2} [\langle f_{h_1}, f_{h_2}\rangle^2_{\mu_{\mathcal{X}}}] = {\mathbb E}_{(h_1, h_2)\sim \chi^2} \Bigg[\left(\int_{\mathcal{X}}f_{h_1}(x)f_{h_2}(x)d\mu_{\mathcal{X}}(x)\right)^2\Bigg] = \\
& {\mathbb E}_{(h_1, h_2)\sim \chi^2} \Bigg[\int_{\mathcal{X}\times \mathcal{X}}f_{h_1}(x)f_{h_2}(x)f_{h_1}(y)f_{h_2}(y)d\mu_{\mathcal{X}}^2(x,y)\Bigg] =\\
& \int_{\mathcal{X}\times \mathcal{X}}{\mathbb E}_{(h_1, h_2)\sim \chi^2}  [f_{h_1}(x)f_{h_1}(y) f_{h_2}(x)f_{h_2}(y)]d\mu_{\mathcal{X}}^2(x,y) = \\
&\int_{\mathcal{X}\times \mathcal{X}}F(x,y)^2 d\mu_{\mathcal{X}}^2(x,y).
\end{split}
\end{equation*}
\end{proof}

Further, we will adopt the notations $F$ and $f_{h}$ from the previous lemma.
Recall that for any $x\in \mathcal{X}$ we have ${\mathbb E}_{Y\sim \mu_\mathcal{Y}}[g(x, Y)]=0$.

\begin{proof} [Proof of Lemma~\ref{aes-rough}] By utilizing ${\mathbb E}_{Y\sim \mu_{\mathcal{Y}}}[g(x,Y)]=0$, the values of $F$ can be expressed as follows:
\begin{equation}\label{representation}
\begin{split}
&F(x,x') = {\mathbb E}_{h\sim \chi}[g(x, h(x))g(x', h(x'))] = \\
&{\mathbb E}_{h\sim \chi}[g(x, h(x))g(x', h(x'))]-{\mathbb E}_{(Y,Y')\sim \mu_{\mathcal{Y}}^2}[g(x, Y)g(x', Y)] \leq \\
&\|\psi_{x,x'}\|_{\mathcal{F}_2}\varepsilon_\mathcal{F}(x,x').
\end{split}
\end{equation}
if $x\ne x'$ and $\psi_{x,x'}(y,y')=g(x, y)g(x', y')$, and 
\begin{equation*}
\begin{split}
&F(x,x) = {\mathbb E}_{h\sim \chi}[g(x, h(x))^2] = {\mathbb E}_{h\sim \chi}[g(x, h(x))^2] - {\mathbb E}_{Y\sim \mu_{\mathcal{Y}}}[g(x, Y)^2]+D_x\leq \\
&\|\psi_{x}\|_{\mathcal{F}_1}\varepsilon_\mathcal{F}(x,x)+D_x,
\end{split}
\end{equation*}
where $D_x = {\mathbb E}_{y\sim \mu_{\mathcal{Y}}}[g(x,Y)^2]$ and $\psi_{x}(y)=g(x, y)^2$.

From the triangle inequality we conclude 
\begin{equation*}
\begin{split}
&\sqrt{ \int_{\mathcal{X}\times \mathcal{X}}F(x,x')^2d\mu_{\mathcal{X}}^2(x,x')}\leq 
\sqrt{\int_{\mathcal{X}\times \mathcal{X}} \varepsilon(x,x')^2\|\psi_{x,x'}\|_{\mathcal{F}_2}^2 d\mu_{\mathcal{X}}^2(x,x')}+\\
&\sqrt{\int_{{\rm diag}(\mathcal{X})}(\|\psi_{x}\|_{\mathcal{F}_1}\varepsilon_\mathcal{F}(x,x)+D_x)^2 d\mu_{\mathcal{X}}^2(x,x') },
\end{split}
\end{equation*}
where ${\rm diag}(\mathcal{X}) = \{(x,x)\mid x\in \mathcal{X}\}$.
Note that 
\begin{equation*}
\begin{split}
&\int_{{\rm diag}(\mathcal{X})}(\|\psi_{x}\|_{\mathcal{F}_1}\varepsilon_\mathcal{F}(x,x)+D_x)^2 d\mu_{\mathcal{X}}^2(x,x')  = \\ 
&{\mathbb P}_{(X,Y)\sim \mu_{\mathcal{X}}^2}[X=Y]{\mathbb E}_{(X,Y)\sim \mu_{\mathcal{X}}^2}[(\|\psi_{X}\|_{\mathcal{F}_1}\varepsilon_\mathcal{F}(X,X)+D_X)^2\mid X=Y].
\end{split}
\end{equation*}
From the latter the statement of Lemma~\ref{aes-rough} is straightforward.
This completes the proof.
\end{proof}

\section{Proofs for Section~\ref{LWE-case}}
\begin{proof}[Proof of Theorem~\ref{uncertainty-one}]  Let us set $\mathcal{F}_i = L_2(\mathcal{Y}^i, \mu_\mathcal{Y}^i )$, $i=1,2$ and use Theorem~\ref{aes-bound}. Recall that the almost pairwise independence for the given $\mathcal{F}$ is measured by the Pearson $\chi^2$ divergence (see example~\ref{Pearson}).
By construction, we have $$\frac{|\{ {\mathbf{k}} \in {\mathbb{Z}}_q^n \mid \langle {\mathbf{k}}, {\mathbf{x}} \rangle = y, \langle {\mathbf{k}}, {\mathbf{x}}' \rangle = y' \}|}{q^n} = q^{-2},$$ and therefore $ \varepsilon_\mathcal{F}({\mathbf{x}}, {\mathbf{x}}') = 0 $, for linearly independent $ {\mathbf{x}}, {\mathbf{x}}' $. 
Recall that we assume ${\mathbf x}\ne {\mathbf 0}$, due to ${\mathbf 0}\notin \mathcal{X}$. 
If $ {\mathbf{x}}' = \lambda {\mathbf{x}} $ and $ \lambda \ne 1$, then $ \frac{|\{ {\mathbf{k}} \in {\mathbb{Z}}_q^n \mid \langle {\mathbf{k}}, {\mathbf{x}} \rangle = y, \lambda \langle {\mathbf{k}}, {\mathbf{x}} \rangle = y' \}|}{q^n} $ equals $ q^{-1}[y' = \lambda y] $, which leads to 
\begin{equation*}
\begin{split}
\varepsilon_\mathcal{F}({\mathbf{x}}, {\mathbf{x}}') = \sqrt{(q^2-q)q^{-2}+q(q^{-1}-q^{-2})^2q^2}\leq \sqrt{q+1}.
\end{split}
\end{equation*}
In summary:
\begin{equation}\label{epsilon}
\varepsilon_\mathcal{F}({\mathbf{x}}, {\mathbf{x}}') = 
\begin{cases}
0, & \text{if rank}([{\mathbf{x}}, {\mathbf{x}}']) = 2; \\
\leq \sqrt{q+1}, & \text{if rank}([{\mathbf{x}}, {\mathbf{x}}']) = 1, {\mathbf{0}} \ne {\mathbf{x}} \ne {\mathbf{x}}'\ne {\mathbf{0}}; \\
0, & \text{if } {\mathbf{x}} = {\mathbf{x}}' \ne {\mathbf{0}}.
\end{cases}
\end{equation}
Let us plug in the expression for $\varepsilon_\mathcal{F}({\mathbf{x}}, {\mathbf{x}}')$ into the RHS of the bound of Theorem~\ref{aes-bound}.  After noting  $\|\phi_{{\mathbf x},{\mathbf x}'}\|_{\mathcal{F}_2} = \sqrt{D_{{\mathbf x}}D_{{\mathbf x}'}}$, we obtain
\begin{equation*}
\begin{split}
&{\mathbb E}_{X,Y\sim \mu_\mathcal{X}}[ \varepsilon_\mathcal{F}(X,Y)^2D_X D_{Y}] \leq 
(q+1)(a^{n}-1)^{-2}\hspace{-40pt}\sum_{{\mathbf x}\ne {\mathbf x}'\in ({\mathbb{Z}}\cap [0,a))^n\setminus \{{\mathbf 0}\}: {\rm rank}({\mathbf x},{\mathbf x}')=1 }D_{{\mathbf x}}D_{{\mathbf x}'}
\leq \\
&(q+1)(a^{n}-1)^{-2}\hspace{-30pt}\sum_{{\mathbf x}\in ({\mathbb{Z}}\cap [0,a))^n\setminus\{{\mathbf 0}\}, \lambda\in {\mathbb{Z}}_q\setminus\{0,1\}: \lambda {\mathbf x}\in ({\mathbb{Z}}\cap [0,a))^n\setminus\{{\mathbf 0}\}}D_{{\mathbf x}}D_{\lambda {\mathbf x}}\leq \\
&(q+1)(a^{n}-1)^{-2}(q-2)\sum_{{\mathbf x}\in ({\mathbb{Z}}\cap [0,a))^n\setminus\{{\mathbf 0}\}}D^2_{{\mathbf x}} = \\
& (a^{n}-1)^{-1}(q+1)(q-2){\mathbb{E}}_{X \sim \mu_\mathcal{X}} \left[ D_X^2 \right].
\end{split}
\end{equation*}
The parameter $\gamma$ equals
\begin{equation*}
\begin{split}
&\gamma = {\mathbb P}_{X,Y\sim \mu_\mathcal{X}}[X=Y]{\mathbb E}_{X,Y\sim \mu_\mathcal{X}}[(\|\phi_{X}\|_{\mathcal{F}_1}\varepsilon_\mathcal{F}(X,X)+D_X)^2\mid X=Y] = \\
& (a^{n}-1)^{-1}{\mathbb E}_{X\sim \mu_\mathcal{X}}[D_X^2],
\end{split}
\end{equation*}
which is smaller that the previous bound. 
After applying Theorem~\ref{aes-bound}, we obtain the needed bound on the variance:
\begin{equation*}
\begin{split}
&{\rm Var}_{h \sim \chi} \left[ \partial_{w_i} {\mathbb{E}}_{X \sim \mu_\mathcal{X}} L(p({\mathbf{w}}, X), h(X)) \right] \leq \\
&{\mathbb{E}}_{X \sim \mu_\mathcal{X}} \left[(\partial_{w_i} p({\mathbf{w}}, X))^2 \right] \times \left( \sqrt{(q+1)(q-2)}+1\right) {\mathbb{E}}_{X \sim \mu_\mathcal{X}} \left[ D_X^2 \right]^{\frac{1}{2}} (a^{n}-1)^{-\frac{1}{2}}.
\end{split}
\end{equation*}
\end{proof}

\section{Proofs for Section~\ref{waves}}
\subsection{Proof of Theorem~\ref{n-more-2}}
Let us define $\mathcal{F}_1 = L_\infty({\mathbb R})$ and $\mathcal{F}_2 = L_\infty({\mathbb R}^{2})$ and let the fractional part of $x\in {\mathbb R}$ be denoted by $\{x\}$. Let $U(\Omega)$ denote the uniform distribution over a set $\Omega$. Then, we have 
\begin{equation*}
\begin{split}
&\varepsilon_{\mathcal{F}}({\mathbf x}, {\mathbf y})  = \sup_{f\in L_\infty({\mathbb R}^{2}): \|f\|_{L_\infty}\leq 1}|{\mathbb E}_{{\mathbf A}\sim \mathcal{N}({\mathbf 0}, R^2 I_n)}[f(\psi({\mathbf A}^\top {\mathbf x}),\psi({\mathbf A}^\top {\mathbf y}))] -\\
&{\mathbb E}_{(X,Y)\sim U([0,1]^2)}[f(\psi(X),\psi(Y))]| = \\
&\sup_{f\in L_\infty({\mathbb R}^{2}): \|f\|_{L_\infty}\leq 1}\hspace{-12pt}|{\mathbb E}_{{\mathbf A}\sim \mathcal{N}({\mathbf 0}, R^2 I_n)}[f(\psi(\{{\mathbf A}^\top {\mathbf x}\}),\psi(\{{\mathbf A}^\top {\mathbf y}\}))] -\\
&{\mathbb E}_{(X,Y)\sim U([0,1]^2)}[f(\psi(X),\psi(Y))]| \leq \\
&\sup_{f\in L_\infty([0,1]^2): \|f\|_{L_\infty}\leq 1}|{\mathbb E}_{{\mathbf A}\sim \mathcal{N}({\mathbf 0}, R^2 I_n)}[f(\{{\mathbf A}^\top {\mathbf x}\},\{{\mathbf A}^\top {\mathbf y}\})] -\\
&{\mathbb E}_{(X,Y)\sim U([0,1]^2)}[f(X,Y)]|.
\end{split}
\end{equation*}
The latter is twice the total variation distance between $(\{X\},\{Y\})$, where $(X,Y)\sim \mathcal{N}({\mathbf 0}, R^2 \Lambda)$, $\Lambda = \begin{bmatrix}
{\mathbf x}^\top {\mathbf x} & {\mathbf x}^\top {\mathbf y} \\
{\mathbf x}^\top {\mathbf y} & {\mathbf y}^\top {\mathbf y}
\end{bmatrix}$, and $U([0,1]^2)$. Let us estimate this distance. 
 
\begin{lemma}\label{wirtinger} Let $(X,Y)\sim \mathcal{N}({\mathbf 0}, \Sigma)$.   Then, the total variation distance between $(\{X\}, \{Y\})$ and $U([0,1]^2)$ is bounded by 
\begin{equation*}
\begin{split}
\frac{C(\Sigma_{11}^{1/2}\Sigma_{22}^{1/2}+|\Sigma_{12}|)(\Sigma_{11}^{1/2}+\Sigma_{22}^{1/2})}{|\Sigma_{11}\Sigma_{22}-\Sigma_{12}^2|},
\end{split}
\end{equation*}
where $C$ is some universal constant.
\end{lemma}
\begin{proof} Let $T(x,y) = (\{x\}, \{y\})$.
Note that for $i,j\in {\mathbb Z}$ and $(x,y)\in [0,1)^2$ the image of $(i+x, j+y)$ under $T$ is $(x,y)$.  Therefore, the joint pdf of $T(X,Y)$ is
\begin{equation*}
\begin{split}
f_{T(X,Y)}(x,y)=\sum_{i,j\in {\mathbb Z}} f_{X,Y}(i+x,j+y),
\end{split}
\end{equation*}
and for any $z\in \{x,y\}$, we have
\begin{equation*}
\begin{split}
\frac{\partial f_{T(X,Y)}(x,y)}{\partial z}=\sum_{i,j\in {\mathbb Z}} \frac{\partial f_{X,Y}(i+x,j+y)}{\partial z}.
\end{split}
\end{equation*}
Let us denote $\|[x,y]^\top\|_1 = |x|+|y|$ and for a pair $f_1,f_2\in L_1([0,1]^2)$ denote $\|[f_1,f_2]^\top\|_{L_1([0,1]^2)} = \|f_1\|_{L_1([0,1]^2)}+\|f_2\|_{L_1([0,1]^2)}$. Then, we have
\begin{equation*}
\begin{split}
&\|\nabla f_{T(X,Y)}(x,y)\|_{L_1([0,1]^2)} \leq \sum_{z\in \{x,y\}}\sum_{i,j\in {\mathbb Z}} \|\frac{\partial f_{X,Y}(i+x,j+y)}{\partial z}\|_{L_1([0,1]^2)} = \\
&\sum_{z\in \{x,y\}}\int_{{\mathbb R}^2} |\frac{\partial f_{X,Y}(x,y)}{\partial z}|dxdy = \int_{{\mathbb R}^2} \|\nabla f_{X,Y}(x,y)\|_1 dxdy.
\end{split}
\end{equation*}
Since $\nabla f_{X,Y}(x,y) =- \Sigma^{-1}\begin{bmatrix}
x\\y
\end{bmatrix} f_{X,Y}(x,y)$, we conclude $\|\nabla f_{X,Y}(x,y)\|_1\leq (\|(\Sigma^{-1})_{\ast,1}\|_1|x|+\|(\Sigma^{-1})_{\ast,2}\|_1|y|)f_{X,Y}(x,y)$, where $(\Sigma^{-1})_{\ast,i}$ is the $i$th column of $\Sigma^{-1}$. Thus,
\begin{equation*}
\begin{split}
&\int_{{\mathbb R}^2} \|\nabla f_{X,Y}(x,y)\|_1 dxdy \leq \|(\Sigma^{-1})_{\ast,1}\|_1{\mathbb E}[|X|]+\|(\Sigma^{-1})_{\ast,2}\|_1{\mathbb E}[|Y|]\leq \\
&\|(\Sigma^{-1})_{\ast,1}\|_1\sqrt{{\mathbb E}[X^2]}+\|(\Sigma^{-1})_{\ast,2}\|_1\sqrt{{\mathbb E}[Y^2]} =\\ 
&\|(\Sigma^{-1})_{\ast,1}\|_1\Sigma_{11}^{1/2}+\|(\Sigma^{-1})_{\ast,2}\|_1\Sigma_{22}^{1/2}.
\end{split}
\end{equation*}
Finally, using Poincaré-Wirtinger inequality~\cite{Ziemer1989}, we conclude that there exists a universal constant $C>0$ such that
\begin{equation*}
\begin{split}
&\int_{[0,1]^2} |f_{T(X,Y)}(x,y)-1|dxdy \leq C\|\nabla f_{T(X,Y)}(x,y)\|_{L_1([0,1]^2)}\leq \\
&\frac{C}{|\Sigma_{11}\Sigma_{22}-\Sigma_{12}^2|}((\Sigma_{22}+|\Sigma_{12}|)\Sigma_{11}^{1/2}+(\Sigma_{11}+|\Sigma_{12}|)\Sigma_{22}^{1/2}) = \\
&\frac{C(\Sigma_{11}^{1/2}\Sigma_{22}^{1/2}+|\Sigma_{12}|)(\Sigma_{11}^{1/2}+\Sigma_{22}^{1/2})}{|\Sigma_{11}\Sigma_{22}-\Sigma_{12}^2|}.
\end{split}
\end{equation*}
The integral $\int_{[0,1]^2} |f_{T(X,Y)}(x,y)-1|dxdy$ is exactly the total variation distance between $(\{X\}, \{Y\})$ and $U([0,1]^2)$.
\end{proof}
From Lemma~\ref{wirtinger} we conclude 
\begin{equation*}
\begin{split}
\varepsilon_{\mathcal{F}}({\mathbf x}, {\mathbf y})  \leq \frac{2C(\|{\mathbf x}\|\|{\mathbf y}\|+|{\mathbf x}^\top {\mathbf y}|)(\|{\mathbf x}\|+\|{\mathbf y}\|)}{R^2(\|{\mathbf x}\|^2\|{\mathbf y}\|^2-({\mathbf x}^\top {\mathbf y})^2)}.
\end{split}
\end{equation*}
Since collision entropy of a continuous distribution is equal to 0, using Theorem~\ref{aes-bound}, we obtain our final bound,
\begin{equation*}
\begin{split}
&{\rm Var}_{h\sim\chi}\big[\partial_{w_i} {\mathbb E}_{X\sim \mu_{\mathcal{X}}} [L(p({\mathbf w},X), h(X))]\big]\leq \\  
&\|\partial_{w_i}p({\mathbf w},\cdot)\|^2_{\mu_{\mathcal{X}}} 
  \sqrt{{\mathbb E}_{(X,Y)\sim \mu_{\mathcal{X}}^2}\big[ \varepsilon_\mathcal{F}(X,Y)^2\|\phi_{X,Y}\|_{\mathcal{F}_2}^2\big]}  \\
&\lesssim \frac{1}{R^2}\|\partial_{w_i}p({\mathbf w},\cdot)\|^2_{\mu_{\mathcal{X}}} \sqrt{{\mathbb E}_{({\mathbf x},{\mathbf y})\sim \mu_{\mathcal{X}}^2}\bigg[\frac{(\|{\mathbf x}\|\|{\mathbf y}\|+|{\mathbf x}^\top {\mathbf y}|)^2(\|{\mathbf x}\|+\|{\mathbf y}\|)^2}{(\|{\mathbf x}\|^2\|{\mathbf y}\|^2-({\mathbf x}^\top {\mathbf y})^2)^2}M_{\mathbf x}^2 M_{\mathbf y}^2\bigg]}.
\end{split}
\end{equation*}

\begin{proof}[Proof of Remark~\ref{singular-kernel}]
Note that $$(\|{\mathbf x}\|\|{\mathbf y}\|+|{\mathbf x}^\top {\mathbf y}|)^2(\|{\mathbf x}\|+\|{\mathbf y}\|)^2\leq 8\|{\mathbf x}\|^2\|{\mathbf y}\|^2(\|{\mathbf x}\|^2+\|{\mathbf y}\|^2).$$ From symmetry between ${\mathbf x}$ and ${\mathbf y}$ we conclude that it is enough to have 
$$
{\mathbb E}_{({\mathbf x},{\mathbf y})\sim \mu_{\mathcal{X}}^2}\bigg[\frac{\|{\mathbf x}\|^4\|{\mathbf y}\|^2}{(\|{\mathbf x}\|^2\|{\mathbf y}\|^2-({\mathbf x}^\top {\mathbf y})^2)^2}\bigg]
$$
bounded. Let $\sigma$ denote the surface area measure on ${\mathbb S}^{n-1} = \{{\mathbf x}\in {\mathbb R}^n \mid \|{\mathbf x}\|=1\}$. Then 
\begin{equation*}
\begin{split}
&\int_{{\mathbb R}^{2n}} \frac{\|{\mathbf x}\|^4\|{\mathbf y}\|^2}{(\|{\mathbf x}\|^2\|{\mathbf y}\|^2-({\mathbf x}^\top {\mathbf y})^2)^2}f({\mathbf x})f({\mathbf y})d{\mathbf x}d{\mathbf y} = \\
&\int\limits_{[0,+\infty)^2}\int\limits_{({\mathbb S}^{n-1})^2} \frac{r_1^4r_2^2}{r_1^4r_2^4(1-({\mathbf x}^\top {\mathbf y})^2)^2}f(r_1 {\mathbf x})f(r_2 {\mathbf y}) d\sigma_1({\mathbf x})d\sigma_2 ({\mathbf y})  r^{n-1}_1dr_1 r^{n-1}_2 dr_2\leq \\
&\int_0^\infty r^{n-1}g(r)dr \int_0^\infty r^{n-3}g(r)dr\int_{{\mathbb S}^{n-1}\times {\mathbb S}^{n-1}} \frac{1}{(1-({\mathbf x}^\top {\mathbf y})^2)^2} d\sigma_1({\mathbf x})d\sigma_2({\mathbf y}).
\end{split}
\end{equation*}
Let $\omega_{n-1}$ denote the surface area of ${\mathbb S}^{n-1}$. The last factor equals $$\omega_{n-1}\int_{{\mathbb S}^{n-1}} \frac{1}{(1-x_1^2)^2} d\sigma_1({\mathbf x}) =
\omega_{n-1}\omega_{n-2}\int_{0}^1 \frac{1}{(1-x_1^2)^2} (1-x_1^2)^{\frac{n-2}{2}}dx_1$$ and is bounded if $n\geq 5$. 

If the pdf is a radial function, the conditions $\int_0^\infty r^{n-1}g(r)dr<\infty$, \\$\int_0^\infty r^{n-3}g(r)dr<\infty$ are equivalent to ${\mathbb E}_{{\mathbf x}\sim \mu_\mathcal{X}}[\|{\mathbf x}\|^{-2}]<\infty$. 
\end{proof}

\subsection{Proof of Theorem~\ref{case-n-1}}
For any function $f: {\mathbb R}^s\to {\mathbb R}$ its Lipschitz coefficient is defined by
\begin{equation*}
\begin{split}
&\|f\|_{C^{0,1}({\mathbb R}^s)}\triangleq\sup_{{\mathbf x}\ne {\mathbf y}\in {\mathbb R}^s}\frac{|f({\mathbf x})-f({\mathbf y})|}{\|{\mathbf x}-{\mathbf y}\|}.
\end{split}
\end{equation*} 
A set of differentiable functions $f: {\mathbb R}^s\to {\mathbb R}$ for which $\|D^\beta f\|_{C^{0,1}({\mathbb R}^s)}<\infty, |\beta|\leq 1$ is denoted by $C^{1,1}({\mathbb R}^s)$. The semi-norm in that space is $\|f\|_{C^{1,1}({\mathbb R}^s)} = \max_{\beta: |\beta| \leq 1}\|D^\beta f\|_{C^{0,1}({\mathbb R}^s)}$. Let us now define $$\mathcal{F} = (C^{1,1}({\mathbb R}), C^{1,1}({\mathbb R}^2)).$$

With an abuse of standard terminology, we define a space of functions of bounded variation on $[0,1]^2$, denoted by ${\rm BV}([0,1]^2)$, as a set of functions $f: [0,1]^2\to {\mathbb R}$ such that $\|f(0,x)\|_{{\rm BV}([0,1])}<\infty$, $\|f(1,x)\|_{{\rm BV}([0,1])}<\infty$, $\|f(x,0)\|_{{\rm BV}([0,1])}<\infty$, $\|f(x,1)\|_{{\rm BV}([0,1])}<\infty$ and $V^2(f)< \infty$ where
\begin{equation*}
\begin{split}
&V^2(f)\triangleq \\
&\sup_{\substack{N,M\in {\mathbb N} \\ 0=x_0\leq \cdots \leq x_N=1 \\0=y_0\leq \cdots \leq y_M=1}} \sum_{i=1}^{N}\sum_{j=1}^{M}|f(x_{i},y_{j})-f(x_{i-1},y_{j})-f(x_{i},y_{j-1})+f(x_{i-1},y_{j-1})|.
\end{split}
\end{equation*}
A natural semi-norm in that space is defined by 
\begin{equation*}
\begin{split}
& \|f\|_{{\rm BV}([0,1]^2)} \triangleq \max\big(\|f(0,x)\|_{{\rm BV}([0,1])}, \|f(1,x)\|_{{\rm BV}([0,1])}, \\
&\|f(x,0)\|_{{\rm BV}([0,1])},\|f(x,1)\|_{{\rm BV}([0,1])}, V^2(f)\big).
\end{split}
\end{equation*}
The following lemma is straightforward.
\begin{lemma}
If $f\in C^{1,1}({\mathbb R}^2)$ and $\psi\in {\rm BV}([0,1])$, then $f(\psi(x), \psi(y))\in {\rm BV}([0,1]^2)$ and, for any $c\in {\mathbb R}$, we have
\begin{equation*}
\begin{split}
&\|f(\psi(x), \psi(c))\|_{{\rm BV}([0,1])}\leq \|f\|_{C^{1,1}} \|\psi\|_{{\rm BV}([0,1])},\\
&\|f(\psi(c), \psi(x))\|_{{\rm BV}([0,1])}\leq \|f\|_{C^{1,1}} \|\psi\|_{{\rm BV}([0,1])},\\
& V^2(f(\psi(x), \psi(y))) \leq \|f\|_{C^{1,1}} \|\psi\|^2_{{\rm BV}([0,1])}.
\end{split}
\end{equation*}
\end{lemma}
\begin{proof}
The first two inequalities follow from $|f(\psi(x), \psi(c))-f(\psi(y), \psi(c))|\leq \|f\|_{C^{1,1}} |\psi(x)-\psi(y)|$. Let us now prove the last one.

Since 
\begin{equation*}
\begin{split}
&f(x',y')-f(x,y) = \int_0^1 (f_{x}(x+t(x'-x),y+t(y'-y))(x'-x)+\\
&f_{y}(x+t(x'-x),y+t(y'-y))(y'-y))dt,
\end{split}
\end{equation*}
we obtain
\begin{equation*}
\begin{split}
&f(\psi(x_{i}),\psi(y_{j}))-f(\psi(x_{i-1}),\psi(y_{j})) = \\& (\psi(x_{i})-\psi(x_{i-1})) \int_0^1 f_{x}(\psi(x_{i-1})+t(\psi(x_{i})-\psi(x_{i-1})),\psi(y_{j}))dt,\\
&f(\psi(x_{i}),\psi(y_{j-1}))-f(\psi(x_{i-1}),\psi(y_{j-1})) =\\ &(\psi(x_{i})-\psi(x_{i-1})) \int_0^1 f_{x}(\psi(x_{i-1})+t(\psi(x_{i})-\psi(x_{i-1})),\psi(y_{j-1}))dt.
\end{split}
\end{equation*}
Therefore,
\begin{equation*}
\begin{split}
&|f(\psi(x_{i}),\psi(y_{j}))-f(\psi(x_{i-1}),\psi(y_{j}))-\\
&f(\psi(x_{i}),\psi(y_{j-1}))+f(\psi(x_{i-1}),\psi(y_{j-1}))|\leq \\
& |\psi(x_{i})-\psi(x_{i-1})| \int_0^1 |f_{x}(\psi(x_{i-1})+t(\psi(x_{i})-\psi(x_{i-1})),\psi(y_{j}))-\\
& f_{x}(\psi(x_{i-1})+t(\psi(x_{i})-\psi(x_{i-1})),\psi(y_{j-1}))|dt \leq \\
&|(\psi(x_{i})-\psi(x_{i-1}))(\psi(y_{j})-\psi(y_{j-1}))| \cdot \|f_x\|_{C^{0,1}}.
\end{split}
\end{equation*}
Therefore,
\begin{equation*}
\begin{split}
&\sup_{\substack{N,M\in {\mathbb N} \\ 0=x_0\leq \cdots \leq x_N=1 \\0=y_0\leq \cdots \leq y_M=1}} \sum_{i=1}^{N}\sum_{j=1}^{M}|f(\psi(x_{i}),\psi(y_{j}))-\\
&f(\psi(x_{i-1}),\psi(y_{j}))-f(\psi(x_{i}),\psi(y_{j-1}))+f(\psi(x_{i-1}),\psi(y_{j-1}))|\leq \\
&\|f_x\|_{C^{0,1}}\sup_{\substack{N,M\in {\mathbb N} \\ 0=x_0\leq \cdots \leq x_N=1 \\0=y_0\leq \cdots \leq y_M=1}} \sum_{i=1}^{N}|\psi(x_{i})-\psi(x_{i-1})|\sum_{j=1}^{M}|\psi(y_{j})-\psi(y_{j-1})| =\\ 
&\|f\|_{C^{1,1}} \|\psi\|^2_{{\rm BV}([0,1])}.
\end{split}
\end{equation*}
Lemma proved. Note that the consequence of the lemma is the following inequality:
\begin{equation*}
\begin{split}
\|f(\psi(x), \psi(y))\|_{{\rm BV}([0,1]^2)} \leq \max(\|\psi\|_{{\rm BV}([0,1])}, \|\psi\|^2_{{\rm BV}([0,1])}) \|f\|_{C^{1,1}} .
\end{split}
\end{equation*}
\end{proof}
From the definition of $\mathcal{F}$ and the previous lemma, we directly obtain
\begin{equation}\label{epsilon-inter}
\begin{split}
&\varepsilon_{\mathcal{F}}(x, y)  =  \sup_{f\in C^{1,1}({\mathbb R}^{2}): \|f\|_{C^{1,1}}\leq 1}|{\mathbb E}_{A\sim \mathcal{N}(0, R^2)}[f(\psi(A x),\psi(A y))] -\\
&{\mathbb E}_{(X,Y)\sim U([0,1]^2)}[f(\psi(X),\psi(Y))]| = \\
&\sup_{f\in C^{1,1}({\mathbb R}^{2}): \|f\|_{C^{1,1}}\leq 1}\hspace{-10pt}|{\mathbb E}_{A\sim \mathcal{N}(0, R^2)}[f(\psi(\{A x\}),\psi(\{A y\}))] -\\
&{\mathbb E}_{(X,Y)\sim U([0,1]^2)}[f(\psi(X),\psi(Y))]| \leq \max(\|\psi\|_{{\rm BV}([0,1])}, \|\psi\|^2_{{\rm BV}([0,1])})\times \\
&\sup_{f\in {\rm BV}([0,1]^2): \|f\|_{{\rm BV}}\leq 1}|{\mathbb E}_{A\sim \mathcal{N}(0, R^2)}[f(\{A x\},\{A y\})] -
{\mathbb E}_{(X,Y)\sim U([0,1]^2)}[f(X,Y)]|.
\end{split}
\end{equation}

Thus, to bound $\varepsilon_{\mathcal{F}}(x, y)$ we need to bound the Integral Probability Metric for the space ${\rm BV}([0,1]^2)$ between two random variables: $(\{A x\},\{A y\}), A\sim \mathcal{N}(0, R^2)$ and $(X,Y)\sim U([0,1]^2)$. To bound this ILP, we need the classical notion of the {\em 
 discrepancy}.

Let ${\mathbf x}_1, \cdots, {\mathbf x}_N\in [0,1]^s$, then
\begin{equation*}
\begin{split}
&D^\ast_N({\mathbf x}_1, \cdots, {\mathbf x}_N) \triangleq \\
&\sup_{0\leq u_1,\cdots, u_s\leq 1}|\frac{|\{i\in [N]\mid {\mathbf x}_i\in [0,u_1)\times \cdots \times [0,u_s)\}|}{N}-\prod_{j=1}^s u_j| .
\end{split}
\end{equation*}
is called the discrepancy of a set $\{{\mathbf x}_1, \cdots, {\mathbf x}_N\}$. This notion can be naturally generalized to probabilistic Borel measures $\nu$ on $[0,1]^s$, by
\begin{equation*}
\begin{split}
D^\ast(\nu) \triangleq \sup_{0\leq u_1,\cdots, u_s\leq 1}| {\mathbb P}_{X\sim \nu}\big[X\in [0,u_1)\times \cdots \times [0,u_s)\big]-\prod_{j=1}^s u_j| .
\end{split}
\end{equation*}
Note that the latter definition is simply the supremum distance between cdfs of $\nu$ and the uniform distribution on $[0,1]^s$.

We will use the following classical result whose proof can be found in~\cite{kuipers2012uniform}.
\begin{lemma}[Koksma-Hlawka inequality] 
Let $f\in {\rm BV}([0,1]^2)$.
Then, for any ${\mathbf x}_1, \cdots, {\mathbf x}_N\in [0,1]^2$, ${\mathbf x}_i = [x_{i1}, x_{i2}]$, we have
\begin{equation*}
\begin{split}
&|\frac{1}{N}\sum_{i=1}^N f({\mathbf x}_i) - \int_{[0,1]^2}f({\mathbf x})d{\mathbf x}|\leq \|f(x,1)\|_{{\rm BV}([0,1])}D^\ast_N(x_{11}, \cdots, x_{N1})+\\
&\|f(1,x)\|_{{\rm BV}([0,1])}D^\ast_N(x_{12}, \cdots, x_{N2})+
V^2(f)D^\ast_N({\mathbf x}_1, \cdots, {\mathbf x}_N).
\end{split}
\end{equation*}
\end{lemma}
It can be generalized directly to the case of measures. 
\begin{lemma}[Koksma-Hlawka inequality for measures] 
Let $f\in {\rm BV}([0,1]^2)$.
Then, for any probabilistic Borel measure $\nu$ on $[0,1]^2$, we have
\begin{equation*}
\begin{split}
&|{\mathbb E}_{X\sim \nu} [f(X)] - \int_{[0,1]^2}f({\mathbf x})d{\mathbf x}|\leq \|f(x,1)\|_{{\rm BV}([0,1])}D^\ast_N(\nu_1)+\\
&\|f(1,x)\|_{{\rm BV}([0,1])}D^\ast_N(\nu_2)+
V^2(f) D^\ast_N(\nu),
\end{split}
\end{equation*}
where $\nu_i$ is a projection of
$\nu$ to the $i$th coordinate.
\end{lemma}
For completeness, below  we give a reduction to the case of discrete measures.
\begin{proof}
Let ${\mathbf x}_1, \cdots, {\mathbf x}_N\sim^{\rm iid} \nu$. Since $f$ has a bounded variation, then it is bounded. Therefore, $f({\mathbf x}_1), \cdots, f({\mathbf x}_N)$ is a random sample with ${\mathbb E}_{X\sim \nu}[f(X)^2]\\< \infty$, and by law of large numbers we have $\frac{1}{N}\sum_{i=1}^N f({\mathbf x}_i)\mathop\to\limits^{N\to +\infty} {\mathbb E}_{X\sim \nu} [f(X)]$ in probability. 

Let us now introduce independent random variables $\sigma_i, 1\leq i\leq N$ drawn from the Rademacher distribution, i.e. ${\mathbb P}[\sigma_i = \pm 1]=\frac{1}{2}$. Using the notion of the Rademacher complexity and Massart lemma, we obtain
\begin{equation*}
\begin{split}
&{\mathbb E}_{{\mathbf x}_i}\sup_{0\leq u_1,u_2\leq 1}|\frac{1}{N} \sum_{i=1}^N \mathds{1}_{{\mathbf x}_i\in [0,u_1)\times  [0,u_2)} - {\mathbb P}_{X\sim \nu}\big[X\in [0,u_1)\times [0,u_2)\big]|\leq \\
&\frac{4}{N}{\mathbb E}_{{\mathbf x}_i} {\mathbb E}_{\sigma_i}\big[\sup_{0\leq u_1, u_2\leq 1} \sum_{i=1}^N \sigma_i \mathds{1}_{{\mathbf x}_i\in [0,u_1)\times [0,u_2)}\big]\leq 4{\mathbb E}_{{\mathbf x}_i}\big[\sqrt{\frac{2\log G}{N}}\big],
\end{split}
\end{equation*}
where $G$ is the number of distinct vectors in $\{{\mathbf z}\in \{0,1\}^N \mid z_i = \mathds{1}_{{\mathbf x}_i\in [0,u_1)\times [0,u_2)}, \\ 0\leq u_1,u_2\leq 1\}$.
Since VC dimension of the set of concepts $\{[0,u_1)\times \cdots \times [0,u_s)\mid 0\leq u_1,\cdots, u_s\leq 1\}$ is $O(s)$, we obtain that the latter expression is $\mathcal{O}(\sqrt{\frac{\log N}{N}})$. I.e., ${\mathbb E}_{{\mathbf x}_i}\big[| D^\ast_N({\mathbf x}_1, \cdots, {\mathbf x}_N)-D^\ast(\nu)|\big] = \mathcal{O}(\sqrt{\frac{\log N}{N}})$. Therefore, we have $D^\ast_N({\mathbf x}_1, \cdots, {\mathbf x}_N)\mathop\to\limits^{N\to +\infty} D^\ast(\nu)$ in probability. Analogously, $$D^\ast_N(x_{1i}, \cdots, x_{Ni})\mathop\to\limits^{N\to +\infty} D^\ast(\nu_i).$$ Thus, using the previous lemma and after sending $N\to+\infty$ we obtain the Koksma-Hlawka inequality for measures.
\end{proof}

The following result is also classical and its proof can be found in~\cite{kuipers2012uniform}. 
\begin{lemma}[Erd{\"o}s-Turán-Koksma inequality] \label{lem:erdos} Let ${\mathbf x}_1, \cdots, {\mathbf x}_N\in [0,1]^s$ and $H\in {\mathbb N}$. Then,
\begin{equation*}
\begin{split}
D^\ast_N({\mathbf x}_1, \cdots, {\mathbf x}_N) \leq C_s \Big(\frac{1}{H}+\sum_{{\mathbf h}: 0<\|{\mathbf h}\|_\infty\leq H}\frac{1}{r({\mathbf h})}\big|\frac{1}{N}\sum_{l=1}^N e^{2\pi {\rm i}  {\mathbf h}^\top {\mathbf x}_l}\big|\Big),
\end{split}
\end{equation*}
where $r({\mathbf h}) = \prod_{j=1}^s\max(|h_j|,1)$ for ${\mathbf h} = [h_1, \cdots, h_s]\in {\mathbb Z}^s$ and  $C_s$ is a constant that only depends on $s$.
\end{lemma}

The proof of the following lemma is absolutely analogous to the proof of Koksma-Hlawka inequality for measures.
\begin{lemma}[Erd{\"o}s-Turán-Koksma inequality for measures] \label{lem:erdos_for_measures} For any probabilistic Borel measure $\nu$ on $[0,1]^s$ and $H\in {\mathbb N}$, we have
\begin{equation*}
\begin{split}
D^\ast_N(\nu ) \leq C_s \Big(\frac{1}{H}+\sum_{{\mathbf h}: 0<\|{\mathbf h}\|_\infty\leq H}\frac{|\hat{\nu}({\mathbf h})|}{r({\mathbf h})}\Big),
\end{split}
\end{equation*}
where $r({\mathbf h}) = \prod_{j=1}^s\max(|h_j|,1)$ for ${\mathbf h} = [h_1, \cdots, h_s]\in {\mathbb Z}^s$,  $C_s$ is a constant that only depends on $s$ and
\begin{equation*}
\begin{split}
\hat{\nu}({\mathbf h}) = \int_{[0,1]^s}e^{2\pi {\rm i} {\mathbf h}^\top {\mathbf x}}d\nu({\mathbf x}).
\end{split}
\end{equation*}
\end{lemma}
Notably, $\hat{\nu}({\mathbf h})$ is the inverse Fourier transform of the measure $\nu$. 

Now we are ready to bound $\varepsilon_{\mathcal{F}}(x, y)$ using the inequality~\eqref{epsilon-inter}. 

\begin{lemma} For fixed $x,y\in {\mathbb R}$, we have 
\begin{equation*}
\begin{split}
&\varepsilon_{\mathcal{F}}(x, y)\leq 
3\max(\|\psi\|_{{\rm BV}([0,1])}, \|\psi\|^2_{{\rm BV}([0,1])})\\
&\min_{H\in {\mathbb N}}\bigg(\frac{1}{H}+\sum_{\substack{a,b\in [-H,H]\cap {\mathbb Z},\\ (a,b)\ne (0,0)}} \frac{e^{-2\pi^2R^2(ax+by)^2}}{\max(|a|,1)\max(|b|,1)}\bigg).
\end{split}
\end{equation*}
\end{lemma}
\begin{proof} 
Let $\nu$ be distributed on $[0,1]^2$ in the same way as $(\{A x\},\{A y\})$ for $A\sim \mathcal{N}(0, R^2)$. Then, $D^\ast_N(\nu_i)\leq D^\ast_N(\nu)$ for $i=1,2$. 
For any $f\in {\rm BV}([0,1]^2)$ such that $\|f\|_{{\rm BV}([0,1]^2)}\leq 1$, Koksma-Hlawka inequality for measures gives us
\begin{equation*}
\begin{split}
&|{\mathbb E}_{A\sim \mathcal{N}(0, R^2)}[f(\{A x\},\{A y\})] -
{\mathbb E}_{(X,Y)\sim U([0,1]^2)}[f(X,Y)]|\leq \\
&(\|f(x,1)\|_{{\rm BV}([0,1])}+
\|f(1,x)\|_{{\rm BV}([0,1])}+
V^2(f))D^\ast_N(\nu)\leq 3 D^\ast_N(\nu).
\end{split}
\end{equation*}

We have
\begin{equation*}
\begin{split}
&\hat{\nu}(a,b) = \int_{-\infty}^\infty\frac{1}{\sqrt{2\pi}}e^{-\frac{t^2}{2}}e^{2\pi {\rm i}(a\{Rtx\}+b\{Rty\})} dt= \\
&\int_{-\infty}^\infty\frac{1}{\sqrt{2\pi}}e^{-\frac{t^2}{2}}e^{2\pi {\rm i}Rt(ax+by)} dt= e^{-2\pi^2R^2(ax+by)^2}.
\end{split}
\end{equation*}
From Erd{\"o}s-Turán-Koksma inequality for measures we obtain
\begin{equation*}
\begin{split}
D^\ast_N(\nu)\leq \min_{H\in {\mathbb N}}\bigg(\frac{1}{H}+\sum_{\substack{a,b\in [-H,H]\cap {\mathbb Z},\\ (a,b)\ne (0,0)}} \frac{e^{-2\pi^2R^2(ax+by)^2}}{\max(|a|,1)\max(|b|,1)}\bigg)
\end{split}
\end{equation*}
Hence,
\begin{equation*}
\begin{split}
&\varepsilon_{\mathcal{F}}(x, y)\leq 3 \max(\|\psi\|_{{\rm BV}([0,1])}, \|\psi\|^2_{{\rm BV}([0,1])}) D^\ast_N(\nu)\leq \\
&3\max(\|\psi\|_{{\rm BV}([0,1])}, \|\psi\|^2_{{\rm BV}([0,1])})\times \\
&\min_{H\in {\mathbb N}}\bigg(\frac{1}{H}+\sum_{\substack{a,b\in [-H,H]\cap {\mathbb Z},\\ (a,b)\ne (0,0)}} \frac{e^{-2\pi^2R^2(ax+by)^2}}{\max(|a|,1)\max(|b|,1)}\bigg).
\end{split}
\end{equation*}
\end{proof}

Now let us make the last step of our proof, i.e. plug in the upper bound on $\varepsilon_{\mathcal{F}}(x, y)$ into the RHS of the inequality~\eqref{aes-uncertainty}. We denote $c({\mathbf x}) = {\mathbb E}_{Y\sim \mu_\mathcal{Y}}\big[\frac{\partial L(p({\mathbf w},{\mathbf x}), Y)}{\partial p}\big]$.
Note that 
\begin{equation*}
\begin{split}
\phi_{{\mathbf x}_1,{\mathbf x}_2}(y_1,y_2) = (\frac{\partial L(p({\mathbf w},{\mathbf x}_1), y_1)}{\partial p}-c({\mathbf x}_1))(\frac{\partial L(p({\mathbf w},{\mathbf x}_2), y_2)}{\partial p}-c({\mathbf x}_2)). 
\end{split}
\end{equation*}
Therefore, $$\|\phi_{{\mathbf x}_1,{\mathbf x}_2}\|_{C^{0,1}({\mathbb R}^2)}\leq M_{{\mathbf x}_2}\sup_{y}|\frac{\partial^2 L(p({\mathbf w},{\mathbf x}_1), y)}{\partial p \partial y}|+M_{{\mathbf x}_1}\sup_{y}|\frac{\partial^2 L(p({\mathbf w},{\mathbf x}_2), y)}{\partial p \partial y}|$$ and  
\begin{equation*}
\begin{split}
&\|\partial_{y_1}\phi_{{\mathbf x}_1,{\mathbf x}_2}\|_{C^{0,1}({\mathbb R}^2)}\leq \\
&M_{{\mathbf x}_2}\sup_{y}|\frac{\partial^3 L(p({\mathbf w},{\mathbf x}_1), y)}{\partial p \partial y^2}|+\sup_{y}|\frac{\partial^2 L(p({\mathbf w},{\mathbf x}_1), y)}{\partial p \partial y}| \sup_{y}|\frac{\partial^2 L(p({\mathbf w},{\mathbf x}_2), y)}{\partial p \partial y}|.
\end{split}
\end{equation*} 
Hence,
\begin{equation*}
\begin{split}
\|\phi_{{\mathbf x}_1,{\mathbf x}_2}\|_{C^{1,1}({\mathbb R}^2)}\lesssim  1,
\end{split}
\end{equation*}
if $|\frac{\partial L(p({\mathbf w},{\mathbf x}), y)}{\partial p}|\lesssim  1$, $|\frac{\partial^2 L(p({\mathbf w},{\mathbf x}), y)}{\partial p \partial y}|\lesssim  1$,$|\frac{\partial^3 L(p({\mathbf w},{\mathbf x}), y)}{\partial p \partial y^2}|\lesssim  1$.

Then, the first term in the RHS of the inequality~\eqref{aes-uncertainty} satisfies
\begin{equation*}
\begin{split}
&\sqrt{{\mathbb E}_{(X,Y)\sim \mu_{\mathcal{X}}^2}\big[ \varepsilon_\mathcal{F}(X,Y)^2\|\phi_{X,Y}\|_{\mathcal{F}_2}^2\big]} \lesssim 3\max(\|\psi\|_{{\rm BV}([0,1])}, \|\psi\|^2_{{\rm BV}([0,1])})\times \\
&\sqrt{{\mathbb E}_{(X,Y)\sim \mu_{\mathcal{X}}^2} \bigg[\min_{H\in {\mathbb N}}\bigg(\frac{1}{H}+\sum_{\substack{a,b\in [-H,H]\cap {\mathbb Z},\\ (a,b)\ne (0,0)}} \frac{e^{-2\pi^2R^2(aX+bY)^2}}{\max(|a|,1)\max(|b|,1)}\bigg)^2\bigg]}.
\end{split}
\end{equation*}
The last factor is bounded by
\begin{equation*}
\begin{split}
\min_{H\in {\mathbb N}}\sqrt{{\mathbb E}_{(X,Y)\sim \mu_{\mathcal{X}}^2}\bigg[\bigg(\frac{1}{H}+\sum_{\substack{a,b\in [-H,H]\cap {\mathbb Z},\\ (a,b)\ne (0,0)}} \frac{e^{-2\pi^2R^2(aX+bY)^2}}{\max(|a|,1)\max(|b|,1)}\bigg)^2\bigg]}.
\end{split}
\end{equation*}
Note that
\begin{equation*}
\begin{split}
&\bigg(\frac{1}{H}+\sum_{\substack{a,b\in [-H,H]\cap {\mathbb Z},\\ (a,b)\ne (0,0)}} \frac{e^{-2\pi^2R^2(aX+bY)^2}}{\max(|a|,1)\max(|b|,1)}\bigg)^2\leq \\
&\frac{1}{H^2}+\frac{2}{H}\sum_{\substack{a,b\in [-H,H]\cap {\mathbb Z},\\ (a,b)\ne (0,0)}} \frac{e^{-2\pi^2R^2(aX+bY)^2}}{\max(|a|,1)\max(|b|,1)}+\\
&\sum_{\substack{a,b\in [-H,H]\cap {\mathbb Z},\\ (a,b)\ne (0,0)}} \frac{1}{\max(|a|,1)\max(|b|,1)}\sum_{\substack{a,b\in [-H,H]\cap {\mathbb Z},\\ (a,b)\ne (0,0)}} \frac{e^{-2\pi^2R^2(aX+bY)^2}}{\max(|a|,1)\max(|b|,1)}.
\end{split}
\end{equation*}
Since $\sum_{\substack{a,b\in [-H,H]\cap {\mathbb Z},\\ (a,b)\ne (0,0)}} \frac{1}{\max(|a|,1)\max(|b|,1)}\lesssim \log^2 (H+1)$, we conclude
\begin{equation*}
\begin{split}
&{\mathbb E}_{(X,Y)\sim \mu_{\mathcal{X}}^2}\bigg[\bigg(\frac{1}{H}+\sum_{\substack{a,b\in [-H,H]\cap {\mathbb Z},\\ (a,b)\ne (0,0)}} \frac{e^{-2\pi^2R^2(aX+bY)^2}}{\max(|a|,1)\max(|b|,1)}\bigg)^2\bigg] \lesssim \\
&\frac{1}{H^2}+\log^2 (H+1)\sum_{\substack{a,b\in [-H,H]\cap {\mathbb Z},\\ (a,b)\ne (0,0)}} \frac{{\mathbb E}_{(X,Y)\sim \mu_{\mathcal{X}}^2}[e^{-2\pi^2R^2(aX+bY)^2}]}{\max(|a|,1)\max(|b|,1)}.
\end{split}
\end{equation*}
Thus,
\begin{equation*}
\begin{split}
&\sqrt{{\mathbb E}_{(X,Y)\sim \mu_{\mathcal{X}}^2}\big[ \varepsilon_\mathcal{F}(X,Y)^2\|\phi_{X,Y}\|_{\mathcal{F}_2}^2\big]} \lesssim 3\max(\|\psi\|_{{\rm BV}([0,1])}, \|\psi\|^2_{{\rm BV}([0,1])})\times \\
&\sqrt{\min_{H\in {\mathbb N}}\bigg(\frac{1}{H^2}+\log^2 (H+1)\sum_{\substack{a,b\in [-H,H]\cap {\mathbb Z},\\ (a,b)\ne (0,0)}} \frac{{\mathbb E}_{(X,Y)\sim \mu_{\mathcal{X}}^2}[e^{-2\pi^2R^2(aX+bY)^2}]}{\max(|a|,1)\max(|b|,1)}\bigg)}.
\end{split}
\end{equation*}
Theorem proved.

\section{Additional experiments}
\begin{figure}[H]
    \centering
    \includegraphics[width=0.45\textwidth]{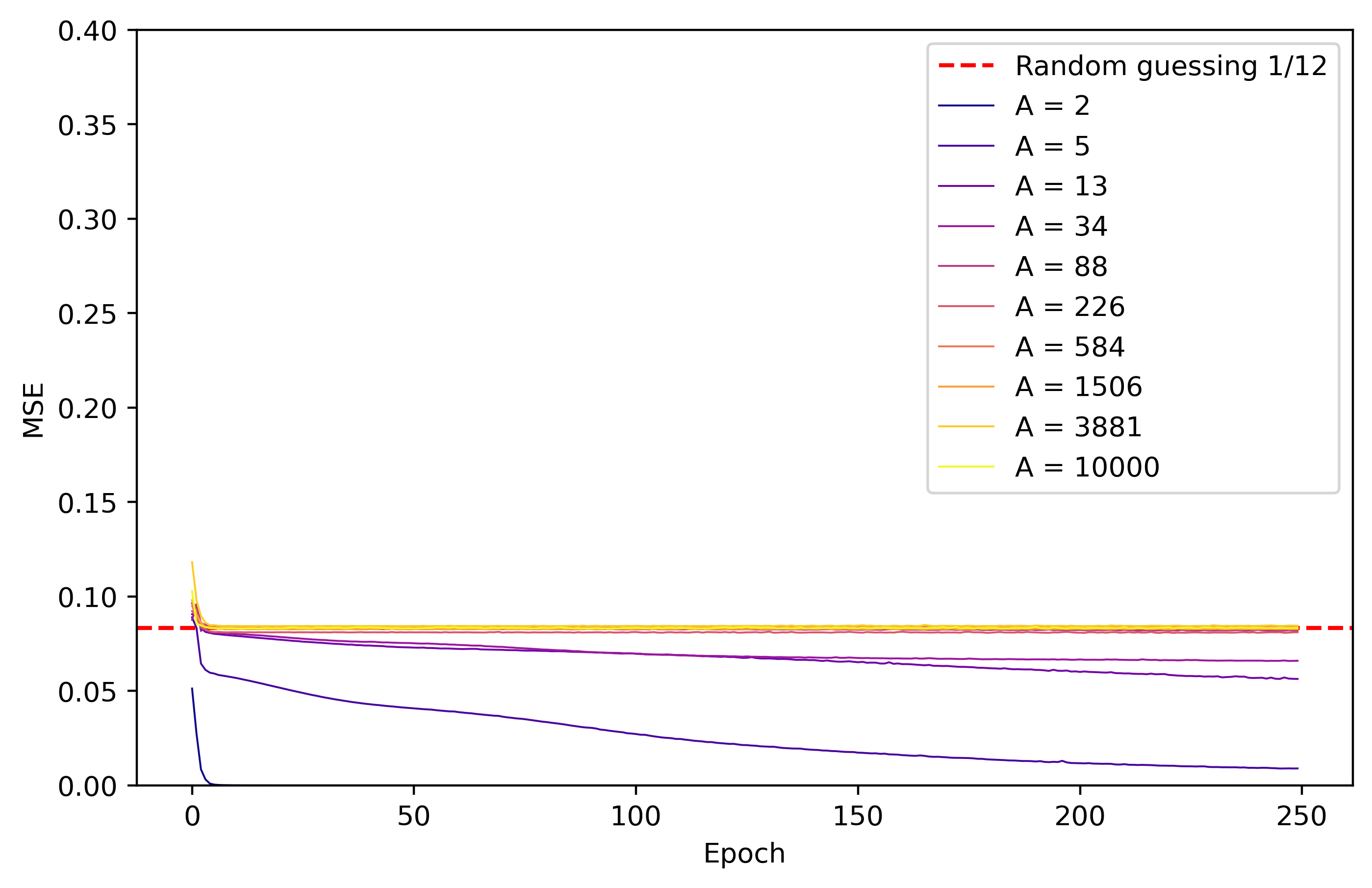} \hfill \includegraphics[width=0.45\textwidth]{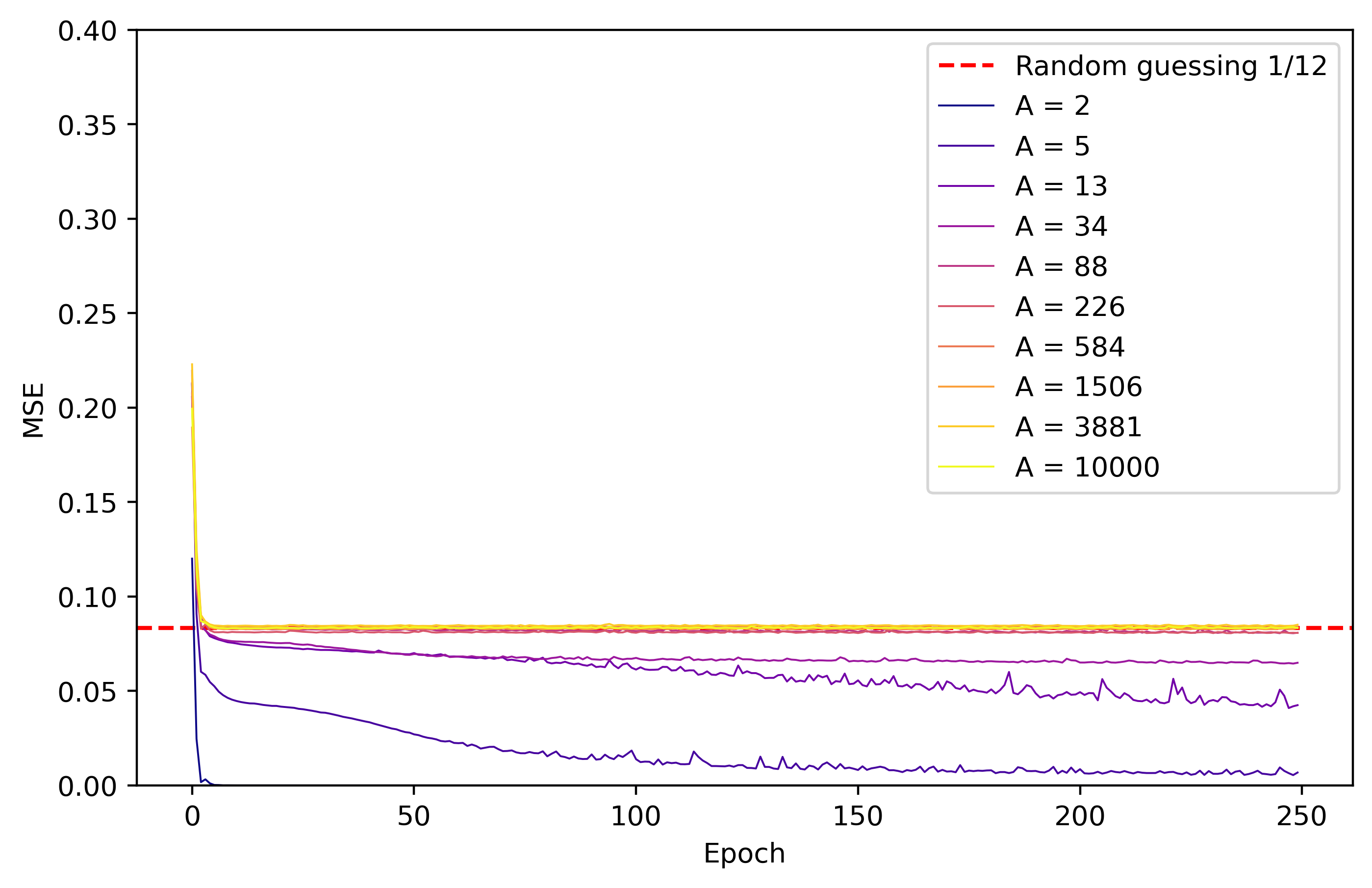} 
\caption{We trained a 3-layer fully connected neural network with ReLU activations to approximate a high-frequency wave $\psi(ax)$ where $\psi(x) = \{x\}$. For each value of $A$,  $a$ was randomly sampled 5 times from the set $\{0,1,\cdots, A-1\}$. The plot shows the mean squared error (MSE) on a test set averaged over these trials as a function of training epochs. The horizontal asymptote represents the MSE of the random guessing, ${\rm MSE} = \frac{1}{12}$. For the first picture and the second pictures the number of neurons on layers are $[1,64,128,1]$ and $[1,640,1280,1]$ (over-parameterized model) respectively. As can be seen there is not much gain in adding more parameters to a model.}\label{over-p}    
\end{figure}
\end{document}